\documentclass[11pt,reqno]{amsart}
\usepackage{algorithm,algorithmicx}
\usepackage{algpseudocode}
\usepackage{bbm}
\usepackage[mathlines,pagewise,left]{lineno}
\usepackage{amssymb}
\usepackage{color}
\usepackage{amsmath}
\usepackage{bcol}
\usepackage{multirow}
\usepackage{caption}
\usepackage[numbers]{natbib}
\usepackage{comment}
\usepackage[toc,page]{appendix}
\usepackage{graphicx} 
\usepackage{booktabs} 
\usepackage{pdflscape}
\usepackage[export]{adjustbox}
\usepackage[font=scriptsize]{subfig}

\usepackage{amsthm}
\newtheorem{proposition}{Proposition}

\newtheorem{theorem}{Theorem}

\theoremstyle{definition}

\theoremstyle{remark}
\newtheorem{remark}{Remark}
\newtheorem{example}{Example}

\newcommand{\R}{\ensuremath{\mathbb{R}}}

\newcommand{\Z}{\ensuremath{\mathbb{Z}}}

\newcommand{\conv}{\ensuremath{\overline{\text{conv}}}}
\newcommand{\be}[1]{\begin{equation}\label{#1}}
\newcommand{\ee}{\end{equation}}

\newcommand{\persp}{\texttt{persp}\ }
\newcommand{\decomp}{\texttt{decomp}\ }
\newcommand{\pairwise}{\texttt{pairwise}\ }
\newcommand{\rev}[1]{#1}
\newcommand{\new}[1]{}
\newcommand{\lasso}{\rev{\texttt{$\ell_1$-approx}}\ }
\newcommand{\del}[1]{{\color{blue}  #1}}

\def\DoubleSpacedXI{\linespread{1.5}}

\DoubleSpacedXI
\usepackage[colorinlistoftodos,prependcaption,textsize=small]{todonotes}
\title[Sparse and smooth signal estimation]{Sparse and Smooth Signal Estimation:\\ Convexification of L0 Formulations}
\author{Alper Atamt\"urk, Andr\'{e}s G\'{o}mez and Shaoning Han}

\thanks{ \noindent \hskip -5mm
A. Atamt\"urk: Department of Industrial Engineering \& Operations Research, University of California, Berkeley, CA 94720.
\texttt{atamturk@berkeley.edu}   \\
A. G\'{o}mez, S. Han: \rev{Daniel J. Epstein Department of Industrial \& Systems Engineering,  University of Southern California, CA 90089. \texttt{gomezand@usc.edu}, \texttt{shaoning@usc.edu}}
}
\begin{document}
	\maketitle
	
	\BCOLReport{18.05}
	
	\begin{abstract}
		\vskip 3mm
		\noindent Signal estimation problems with smoothness and sparsity priors can be naturally modeled as quadratic optimization with $\ell_0$-``norm" constraints. Since such problems are non-convex and hard-to-solve, the standard approach is, instead, to tackle their convex surrogates based on $\ell_1$-norm relaxations. 
		 In this paper, we propose new iterative \rev{(convex)} 
		 conic quadratic relaxations that
		 exploit not only the $\ell_0$-``norm" terms, but also the fitness and smoothness functions. The iterative convexification approach substantially closes the gap between the $\ell_0$-``norm" and its $\ell_1$ surrogate. 
		 \rev{These stronger relaxations lead to} significantly better estimators than $\ell_1$-norm approaches \rev{and also allow one to utilize affine sparsity priors.} In addition, the parameters of the model and the resulting estimators are easily interpretable.
		 Experiments 
		 \rev{with a tailored Lagrangian decomposition method}
		 indicate that the proposed iterative convex relaxations \rev{yield solutions within 1\% of the exact $\ell_0$ approach, and can tackle instances with up to 100,000 variables under one minute.}  \\
		 
		 \noindent
		\textbf{Keywords} Mixed-integer quadratic optimization, conic quadratic optimization, perspective formulation,  sparsity. \\
	\end{abstract}

\begin{center}
	November 2018\rev{; January 2020}\new{; August 2020}
\end{center}

\pagebreak

\section{Introduction}

Given nonnegative data $y\in \R_+^n$ corresponding to a noisy realization of an underlying signal, we consider the problem of removing the noise and recovering the original, uncorrupted signal $y^*$. A successful recovery of the signal requires exploiting \emph{prior} knowledge on the structure and characteristics of the signal effectively.

A common prior knowledge on the underlying signal is \emph{smoothness}. Smoothing considerations can be incorporated in denoising problems through quadratic penalties for deviations in successive estimates \citep{poggio1985computational}. In particular, denoising of a smooth signal can be done by solving an optimization problem of the form  
\begin{equation}
\label{eq:smooth}
\min_{x\in \R_+^n}\|y-x\|_2^2+\lambda\|Px\|_2^2,
\end{equation}
where $x$ corresponds to the estimation for $y^*$, $\lambda>0$ is a smoothing regularization parameter, $P\in \R^{m\times n}$ is a linear operator, the estimation error term $\|y-x\|_2^2$ measures the \emph{fitness} to data, and the quadratic penalty term $\|Px\|_2^2$ models the smoothness considerations. In its simplest form
\begin{equation}
\label{eq:mrf}
\|Px\|_2^2=\sum_{\rev{\{i,j\}}\in A}(x_i-x_j)^2,
\end{equation}
where $A$ encodes the notion of adjacency, e.g., consecutive observations in a time series or adjacent pixels in an image. If $P$ is given according to \eqref{eq:mrf}, then problem \eqref{eq:smooth} is a convex \emph{Markov Random Fields} problem \citep{hochbaum2001efficient} or \emph{metric labeling problem} \citep{kleinberg2002approximation}, commonly used in the image segmentation context \citep{boykov2001fast,kolmogorov2004energy} \new{and in clustering \citep{hochbaum2018adjacency},} for which efficient combinatorial algorithms exist.
\ignore{ Another well-known special case of \eqref{eq:smooth}, arising in time-series analysis, corresponds to
\begin{equation}
\label{eq:hodrickPrescott}
\|Px\|_2^2=\sum_{i=2}^{n-1}(x_{i-1}-2x_i+x_{i+1})^2;
\end{equation}
in this case, \eqref{eq:smooth} corresponds to the \emph{Hodrick-Prescott filter} \citep{baxter1999measuring,hodrick1997postwar,leser1961simple}, popular in the economics literature, supported by standard packages for data analysis such as \texttt{SAS} and \texttt{R}, and solvable in $O(n)$ operations.} Even in its general form, \eqref{eq:smooth} is a convex quadratic optimization, for which a plethora of efficient algorithms exist.

Another naturally occurring signal characteristic is \emph{sparsity}, i.e., the underlying signal differs from a base value in only a small proportion of the indexes. Sparsity arises in diverse application domains including medical imaging \citep{lustig2007sparse}, genomic studies \citep{huang2008adaptive}, face recognition \citep{yang2010fast}, and is at the core of \emph{compressed sensing} methods \citep{donoho2006compressed}. In fact, the ``bet on sparsity" principle \citep{hastie2001elements} calls for systematically assuming sparsity in high-dimensional statistical inference problems. 
Sparsity constraints can be modeled using the $\ell_0$-``norm"\footnote{The so-called $\ell_0$-``norm" is not a proper norm as it violates homogeneity.}, leading to estimation problems of the form
\begin{equation}
\label{eq:denoising}
\min_{x\in\R_+^n}\|y-x\|_2^2+\lambda\sum_{\rev{\{i,j\}}\in A}(x_i-x_j)^2 \text{ subject to }\|x\|_0\leq k,
\end{equation}
where $k\in \Z_+$ is a target sparsity and $\|x\|_0=\sum_{i=1}^n\mathbbm{1}_{x_i\neq 0}$, where $\mathbbm{1}_{(\cdot)}$ is the indicator function equal to $1$ if $(\cdot)$ is true and equal to $0$ otherwise. 
\rev{In addition, the indicators can also be used to model \emph{affine sparsity constraints} \cite{dong2019integer,dong2019structural}, enforcing more sophisticated priors than simple sparsity; see Section~\ref{sec:syntData} for an illustration. }

Unlike \eqref{eq:smooth}, problem \eqref{eq:denoising} is \emph{non-convex} and hard-to-solve exactly. 
The regularized version of \eqref{eq:denoising}, given by 
\begin{equation}
\label{eq:lagrangean}
\min_{x\in\R_+^n}\|y-x\|_2^2+\lambda\sum_{\rev{\{i,j\}}\in A}(x_i-x_j)^2 +\mu \|x\|_0
\end{equation}
with $\mu\geq 0$, has received (slightly) more attention. 
Problem \eqref{eq:lagrangean} corresponds to a Markov Random Fields problem with non-convex deviation functions \citep[see][]{ahuja2004cut,hochbaum2013multi}, for which a pseudo-polynomial combinatorial algorithm of complexity $O\left(\frac{|A|n}{\epsilon^2}\log\left(\frac{n^2}{\epsilon |A|}\right)\right)$ exists, where $\epsilon$ is a precision parameter \rev{and $|A|$ is the cardinality of set $A$}; to the best of our knowledge, this algorithm has not been implemented to date. More recently, in the context of signal denoising, \citet{bach2016submodular} proposed another pseudo-polynomial algorithm of complexity $O\left(\left(\frac{n}{\epsilon}\right)^3\log\left(\frac{n}{\epsilon} \right)\right)$, and demonstrated its performance for instances with $n=50$. The aforementioned algorithms rely on a discretization of the $x$ variables, and their performance depends on how precise the discretization (given by the parameter $\epsilon$) is. 
Finally, a recent result of \citet{atamturk2018strong} on quadratic optimization with M-matrices and indicators imply that \eqref{eq:lagrangean} is equivalent to a submodular minimization problem, which leads to a strongly polynomial-time algorithm of complexity $O(n^7)$. The high complexity by a blackbox submodular minimization algorithm precludes its use except for small instances. No polynomial-time algorithm is known for the constrained problem \eqref{eq:denoising}. 

In fact, problems \eqref{eq:denoising} and \eqref{eq:lagrangean} are rarely tackled directly. One of the most popular techniques used to tackle signal estimation problems with sparsity consists of replacing the non-convex term $\|x\|_0$ with the convex $\ell_1$-norm, $\|x\|_1=\sum_{i=1}^n|x_i|$, see Section~\ref{sec:l1norm} for details. The resulting optimization problems with the $\ell_1$-norm can be solved very efficiently, even for large instances; however, the $\ell_1$ problems are often weak relaxations of the exact $\ell_0$ problem \eqref{eq:denoising}, and the estimators obtained may be poor, as a consequence. Alternatively, there is a increasing effort for solving the mixed-integer optimization (MIO) \eqref{eq:denoising} exactly using enumerative techniques, see Section~\ref{sec:MIO}. While the recovered signals are indeed high quality, \rev{exact} MIO approaches to-date \rev{require at least a few days to solve instances with $n\geq 1,000$}, and are inadequate to tackle many realistic instances as a consequence.

\subsection*{Contributions and outline} In this paper, we discuss how to bridge the gap between the easy-to-solve $\ell_1$ approximations and the often intractable $\ell_0$ problems in a convex optimization framework. Specifically, we construct a set of \textit{iterative convex relaxations} for problems \eqref{eq:denoising} and \eqref{eq:lagrangean} with increasing strength. These convex relaxations are considerably stronger than the $\ell_1$ relaxation, and also significantly improve and generalize other existing convex relations in the literature, including the \emph{perspective relaxation} (see Section~\ref{sec:perspective}) and recent convex relaxations obtained from simple pairwise quadratic terms (see Section~\ref{sec:strong2var}). The strong convex relaxations can be used 
to obtain high quality, if not optimal, solutions for \eqref{eq:denoising}--\eqref{eq:lagrangean}\rev{, resulting} in better performance than the existing methods\rev{; in our computations, solutions to instances with $n=1,$$000$ are obtained with off-the-shelf convex solvers within seconds}. \rev{For additional scalability, we give an easy-to-parallelize tailored Lagrangian decomposition method that solves instances with $n=100,$$000$ under one minute.} Finally, the proposed formulations are amenable to \emph{conic quadratic} optimization techniques, thus can be tackled using off-the-shelf solvers, resulting in several \rev{advantages}: \emph{(i)}~\rev{the methods described here will benefit from the continuous improvements of conic quadratic optimization solvers}; \emph{(ii)}~the proposed approach is flexible, as it can be used to tackle either \eqref{eq:denoising} or \eqref{eq:lagrangean}, \rev{as well as general affine sparsity constraints, by simply changing the objective or adding constraints}.

 Figure~\ref{fig:nonnegSynt10Intro} illustrates the performance of the $\ell_1$-norm estimator and the proposed \rev{strong convex} estimators \rev{for an instance with $n=1,$$000$}. \rev{The new convex estimator, depicted in Figure~\ref{fig:nonnegSynt10Intro}(C), requires only one second to solve; the convex estimator enhanced with additional priors in Figure~\ref{fig:nonnegSynt10Intro}(D) is solved under five seconds.}
 
 \ignore{
 \begin{figure}[h!]
 	\subfloat[True signal (gray) and noisy observations .]{\includegraphics[width=0.50\textwidth,trim={10.5cm 6cm 10.5cm 6cm},clip]{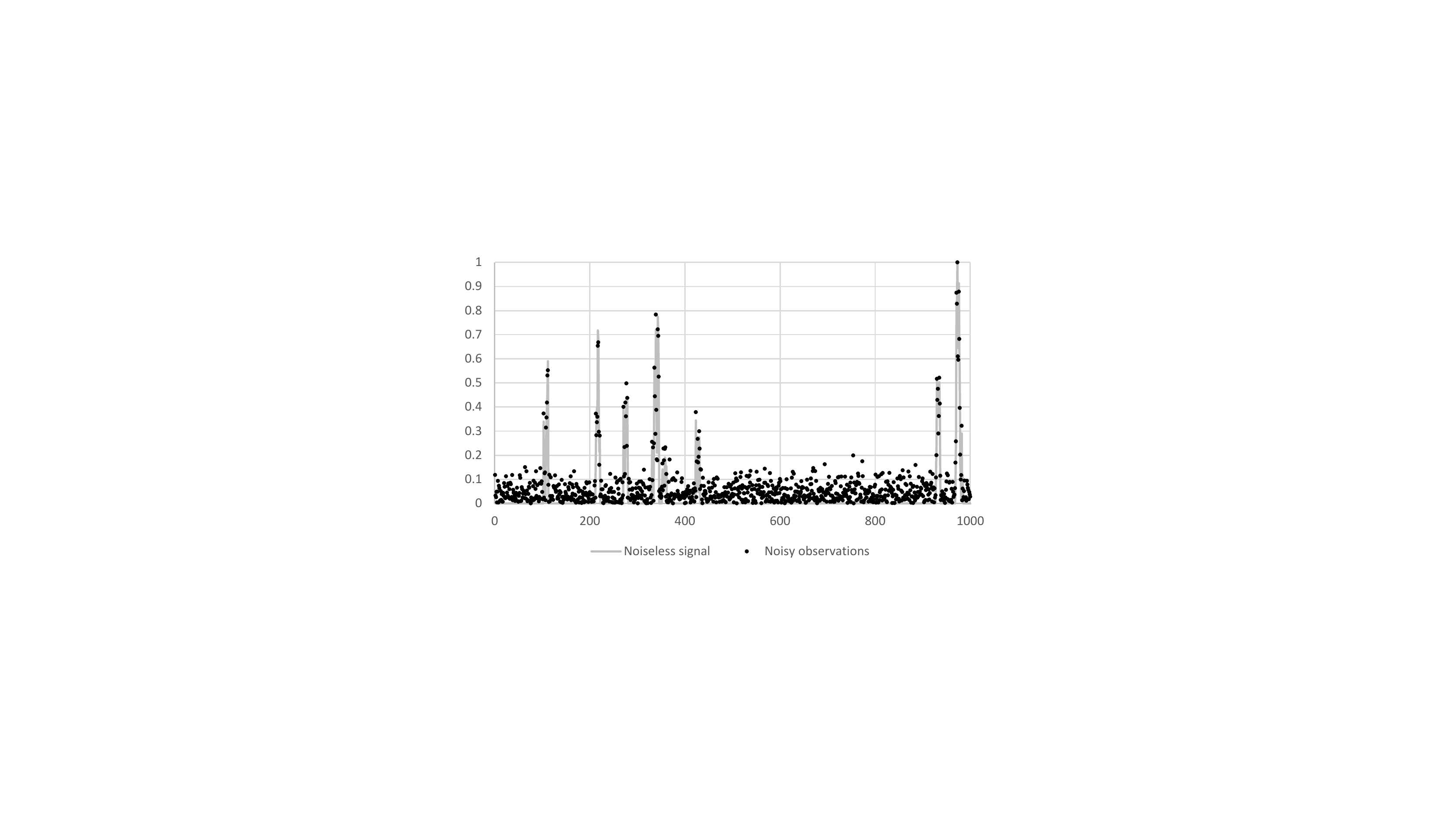}}\hfill
 	\subfloat[\lasso results in dense and shrunk signals with many ``false positives".]{\includegraphics[width=0.50\textwidth,trim={10.5cm 6cm 10.5cm 6cm},clip]{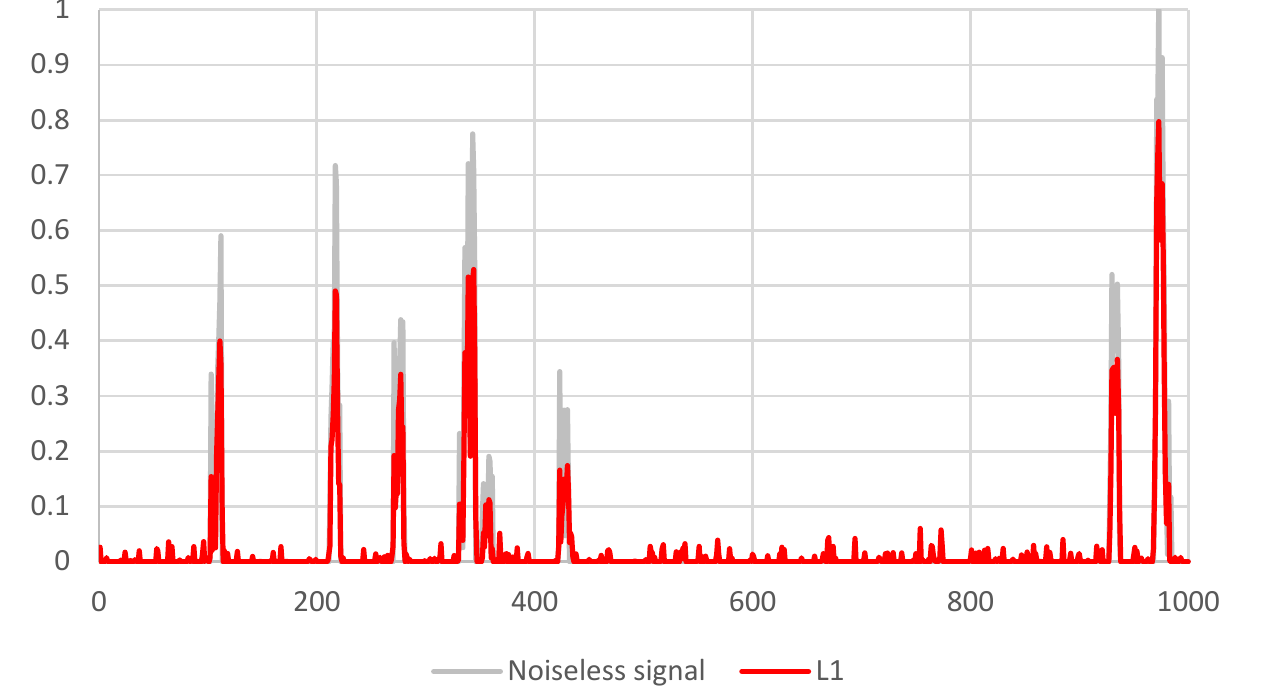}} \newline
 	\subfloat[New convex relaxation of \eqref{eq:denoising} yields a better sparse estimator, but has q few ``false positives." ]{\includegraphics[width=0.50\textwidth,trim={10.5cm 6cm 10.5cm 6cm},clip]{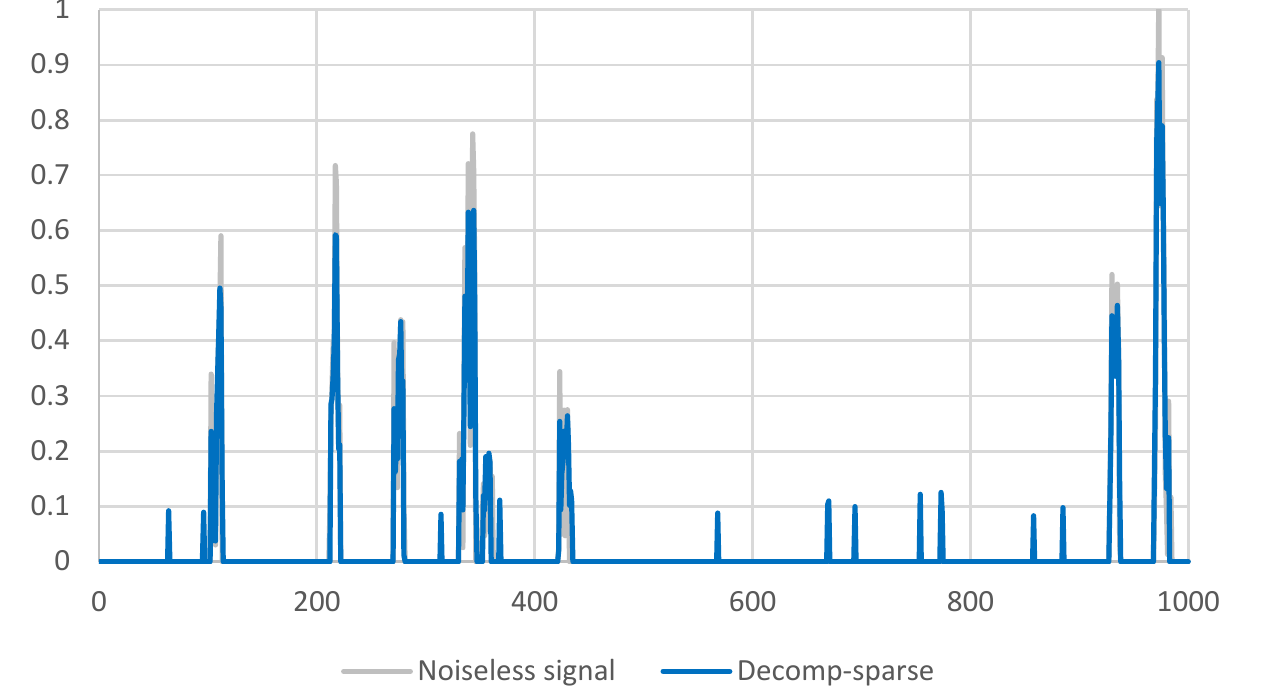}}\hfill
 	\subfloat[Priors can easily be incorporated into proposed relaxations, yielding a near-perfect estimator in this case.]{\includegraphics[width=0.50\textwidth,trim={10.5cm 6cm 10.5cm 6cm},clip]{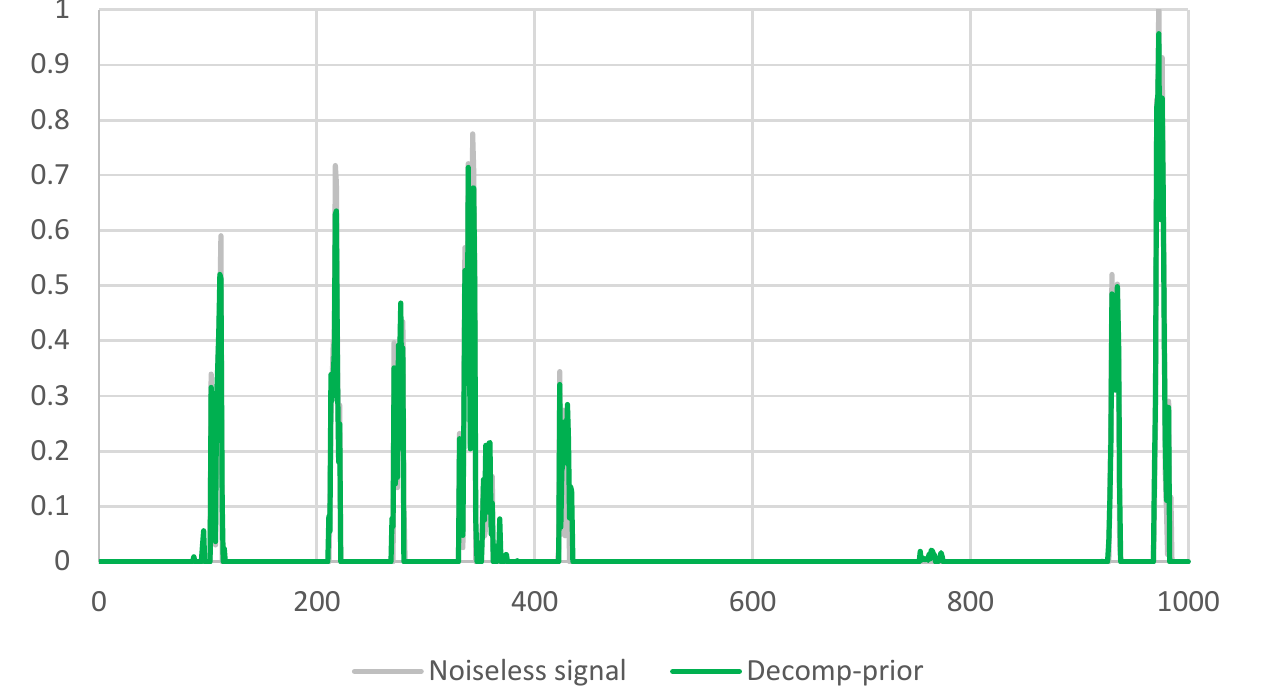}}
 \end{figure}
}
 \begin{figure}[h!]
 	\subfloat[True signal and noisy observations.]{\includegraphics[width=0.50\textwidth,trim={10.5cm 6cm 10.5cm 6cm},clip]{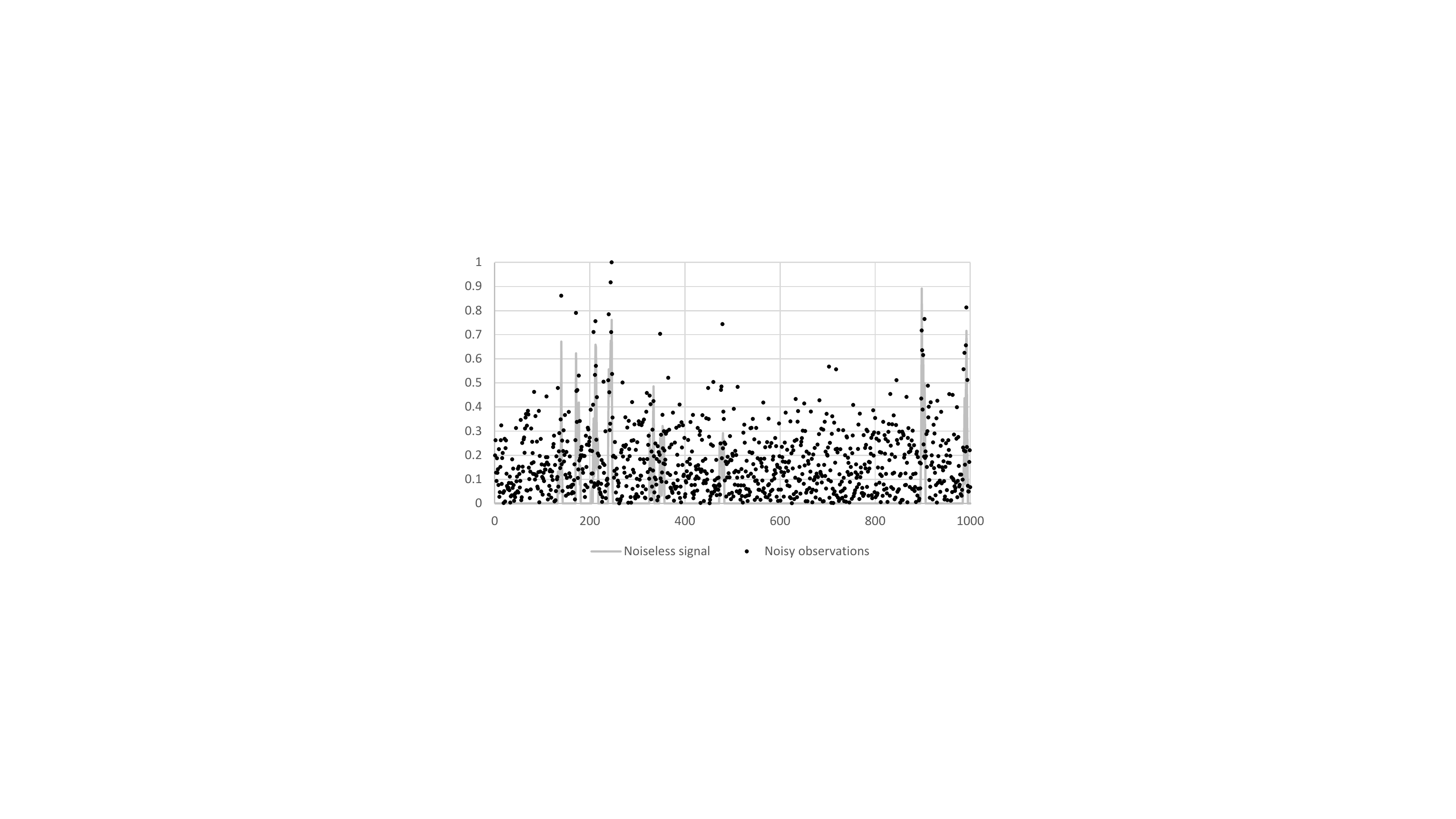}}\hfill
 	\subfloat[\lasso results in dense and shrunk estimators with many ``false positives."]{\includegraphics[width=0.50\textwidth,trim={10.5cm 6cm 10.5cm 6cm},clip]{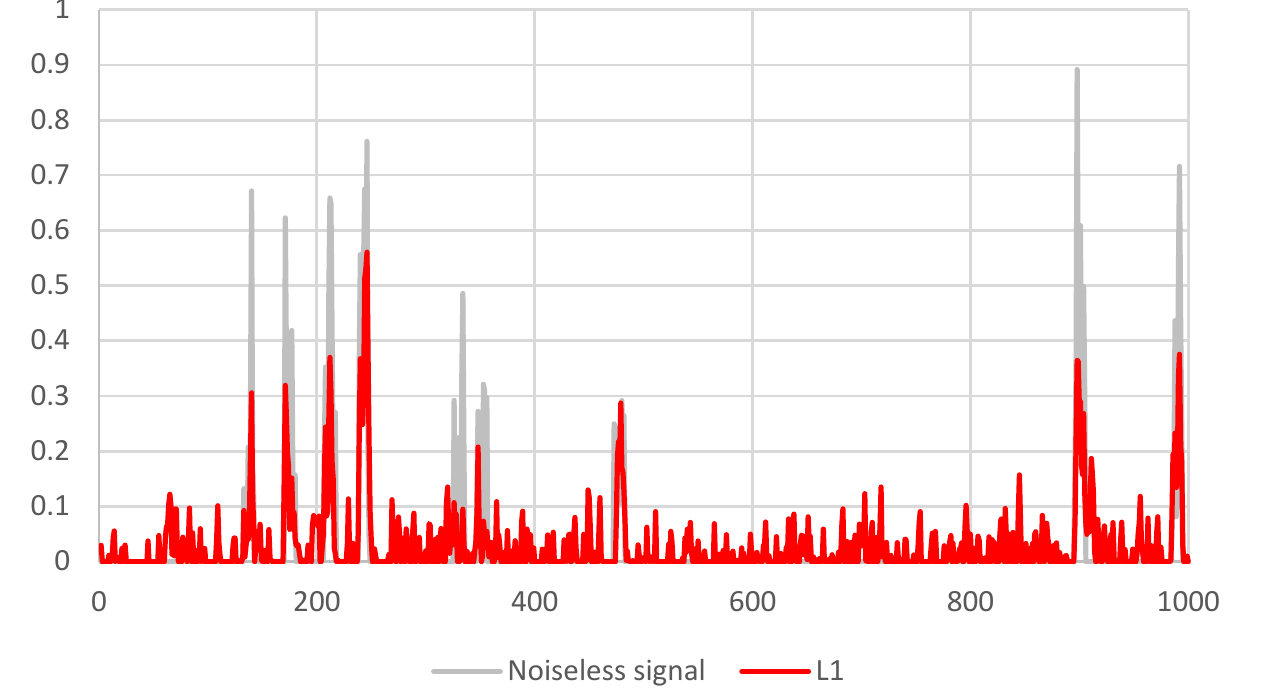}} \newline
 	\subfloat[New strong convex formulation yields better sparse estimators with few ``false positives." ]{\includegraphics[width=0.50\textwidth,trim={10.5cm 6cm 10.5cm 6cm},clip]{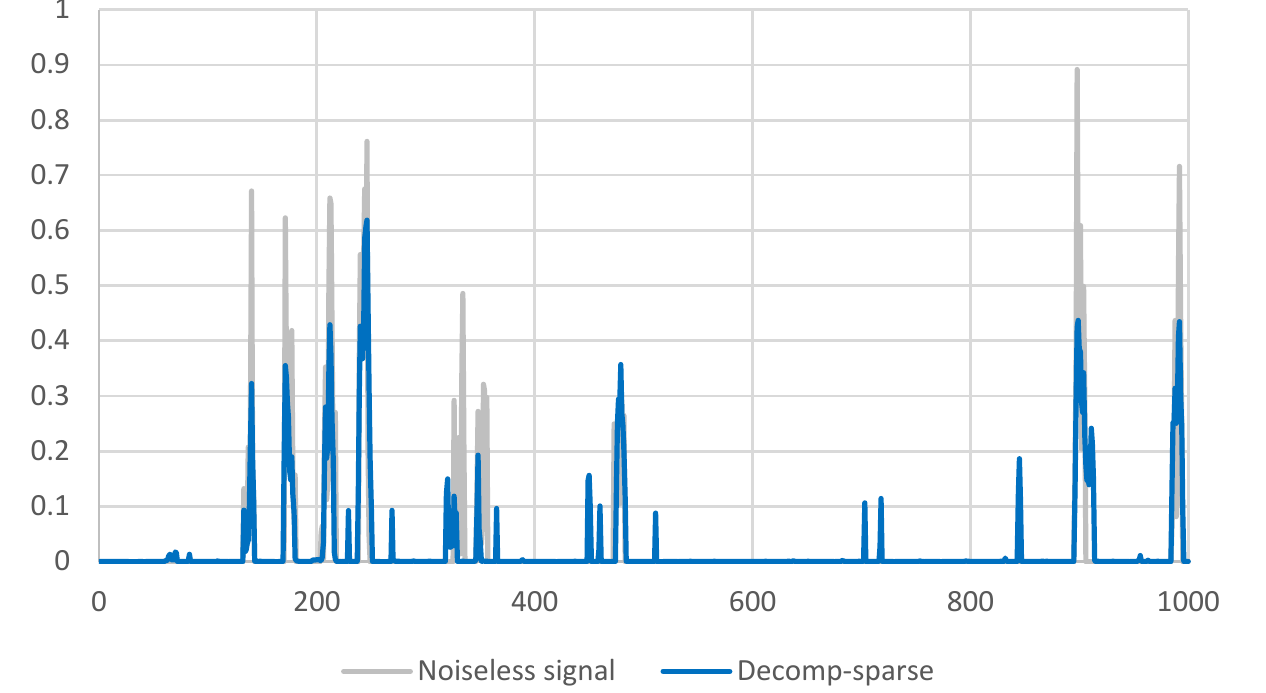}}\hfill
 	\subfloat[Incoporating additional priors further improves the estimators, matching the sparsity pattern of the signal.]{\includegraphics[width=0.50\textwidth,trim={10.5cm 6cm 10.5cm 6cm},clip]{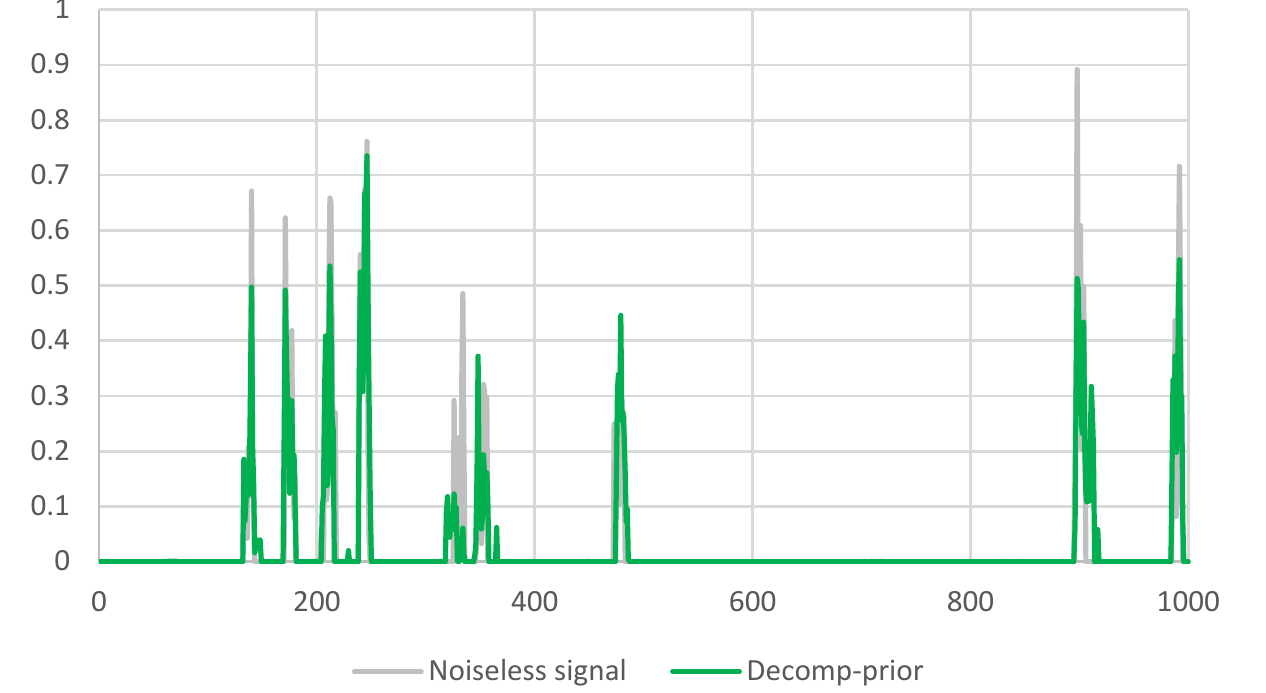}}
 	\caption{Estimators from \rev{$\ell_1$-approximation} and the new strong convex formulations (decomp) for signal denoising.}
 	\label{fig:nonnegSynt10Intro}
 \end{figure}

The rest of the paper is organized as follows. In Section~\ref{sec:background} we review the relevant background for the paper. In Section~\ref{sec:Mmatrix} we introduce the strong iterative convex formulations for \eqref{eq:denoising}--\eqref{eq:lagrangean}. In Section~\ref{sec:implementation} we \rev{give conic quadratic extended reformulation of the model and describe a scalable
Lagrangian decomposition method to solve it.}
 In Section~\ref{sec:computations} we test the performance of the methods \rev{from a computational and statistical perspective}, and in Section~\ref{sec:conclusions} we conclude the paper with a few final remarks.

\subsection*{Notation} Throughout the paper, we adopt the following convention for division by $0$: given $a\geq 0$, $a/0=\infty$ if $a>0$ and $a/0=0$ if $a=0$. For a set $X\subseteq \R^n$, \rev{let $\text{conv}(X)$ denote the convex hull of $X$ and  $\conv(X)$ the closure of $\text{conv}(X)$}. \rev{Given two matrices $Q$, $R$ of the same dimensions, we denote by $\langle Q,R\rangle$ the inner product of $Q$ and $R$.}

\section{Background}
\label{sec:background}
In this section, we review formulations relevant to our discussion.
First we review the usual $\ell_1$-norm approximation (Section~\ref{sec:l1norm}), next we discuss MIO formulations (Section~\ref{sec:MIO}), then we review the perspective reformulation, a standard technique in the MIO literature, (Section~\ref{sec:perspective}), and finally pairwise convex relaxations that were recently proposed (Section~\ref{sec:strong2var}). 
\subsection{L1-norm approximations}
\label{sec:l1norm}
A standard technique for signal estimation problems with sparsity is to replace the $\ell_0$-norm with the $\ell_1$-norm in \eqref{eq:denoising}, leading to the convex optimization problem
\begin{equation}
\label{eq:lasso}
(\rev{\texttt{$\ell_1$-approx}}) \  \ \ \ \ \ \ \min_{x\in \R_+^n}\|y-x\|_2^2+\lambda\new{\sum_{\{i,j\}\in A}} (x_i-x_j)^2\text{ subject to }\|x\|_1\leq k.
\end{equation}
The $\ell_1$-norm approximation was proposed by \citet{tibshirani1996regression} in the context of sparse linear regression, and is often referred to as lasso. The main motivation for the $\ell_1$-approximation is that the $\ell_1$-norm is the convex $p$-norm \new{closest} to the $\ell_0$-norm. In fact, for $L=\left\{x\in [0,1]^n:\|x\|_0\leq 1\right\}$, it is easy to show that $\text{conv}(L)= \left\{x\in [0,1]^n:\|x\|_1\leq 1\right\}$; therefore, the $\ell_1$-norm approximation is considered to be the best possible convex relaxation of the $\ell_0$-norm. 

\rev{The} $\ell_1$-approximation is currently the most commonly used approach for sparsity \citep{hastie2015statistical}. It has been applied to a variety of signal estimation problems including signal decomposition and spike detection \citep[e.g.,][]{chen2001atomic,friedrich2017fast,vogelstein2010fast,lin2014}, and pervasive
 in the compressed sensing literature \citep{candes2008introduction,candes2008enhancing,donoho2006stable}. A common variant is the fused lasso \citep{tibshirani2005sparsity}, which involves a sparsity-inducing term of the form $\sum_{i=1}^{n-1}|x_{i+1}-x_i|$; the fused lasso was further studied in the context of signal estimation \cite{rinaldo2009properties}, and is often used for digital imaging processing under the name of \emph{total variation denoising} \citep{rudin1992nonlinear,vogel1996iterative,padilla2017}. Several other generalizations of the $\ell_1$-approximation exist \citep{tibshirani2011regression}, including the elastic net \citep{zou2005regularization,nevo2017}, the adaptive lasso \citep{zou2006adaptive}, the group lasso \citep{bach2008,qin2012} and the smooth lasso \citep{hebiri2011smooth}; related $\ell_1$-norm techniques have also been proposed for signal estimation, see \cite{kim2009ell,mammen1997locally,tibshirani2014adaptive}.
 The generalized lasso \citep{tibshirani2011solution} utilizes the regularization term $\|Ax\|_1$ and is also studied in the context of signal approximation.

Despite its widespread adoption, \rev{the} $\ell_1$-approximation has several drawbacks. First, the $\ell_1$-norm term may result in excessive \emph{shrinkage} of the estimated signal, which is undesirable in many contexts \cite{zhang2010nearly}. Additionally, \rev{the $\ell_1$-approximation} may struggle to achieve sparse estimators --- in fact, solutions to \eqref{eq:lasso} are often dense, and achieving a target sparsity of $k$ requires using a parameter $\hat k<< k$, inducing additional bias on the estimators.  As a consequence, desirable theoretical performance of the $\ell_1$-approximation can only be established under stringent conditions \cite{rinaldo2009properties,shen2013constrained}, which may not be satisfied in practice. Indeed, $\ell_1$-approximations have been shown to perform rather poorly in a variety of contexts, e.g., see \cite{jewell2017exact,miller2002subset}. To overcome the aforementioned drawbacks, several non-convex approximations have been proposed \citep{frank1993statistical,hazimeh2018fast,mazumder2011sparsenet,zhang2014lower,zheng2014high}; more recently, there is also an increasing effort devoted to enforcing sparsity directly with $\ell_0$ regularization using enumerative MIO approaches.

\subsection{Mixed-integer optimization}
\label{sec:MIO}
Signal estimation problems with sparsity can be naturally modeled as a mixed-integer quadratic optimization (MIQO) problem. Using indicator variables $z\in \{0,1\}^n$ such that $z_i=\mathbbm{1}_{x_i\neq 0}$ for all $i=1,\ldots,n$, problem \eqref{eq:denoising} can be formulated as 
\begin{subequations}
	\label{eq:denoisingMIO}
	\begin{align}
	\min\;&\sum_{i=1}^n (y_i-x_i)^2+\lambda \sum_{\rev{\{i,j\}}\in A}(x_i-x_j)^2\label{eq:objective}\\
	\text{s.t.}\;
	&x_i(1-z_i)=0\label{eq:complementary}\\
	&\rev{z\in C\subseteq \{0,1\}^n} \label{eq:cardinality}\\
	& x\in \R_+^n.
	\end{align}
\end{subequations}
\rev{If $C$ is defined by a $k$-sparsity constraint, i.e., $C=\left\{z\in \{0,1\}^n: \|z\|_1\leq k\right\}$, then problem \eqref{eq:denoisingMIO} is the $\ell_0$ analog of \eqref{eq:lasso}. More generally, $C$ may be defined by 
other logical (affine sparsity) constraints, which allow the inclusion of additional priors in the inference problem.}
In this formulation,
the non-convexity of the $\ell_0$ regularizer is captured by the complementary constraints \eqref{eq:complementary} and the binary constraints \rev{encoded by set $C$}. Constraints \eqref{eq:complementary} can be alternatively formulated with the so-called ``big-$M$" constraints with a sufficiently large positive number $u$,
\begin{equation}
\label{eq:bigM}
x_i(1-z_i)=0 \text{ and }z_i\in \{0,1\} \Leftrightarrow x_i\leq u z_i \text{ and }z_i\in \{0,1\}.
\end{equation}
For the signal estimation problem \eqref{eq:denoisingMIO}, $u=\|y\|_\infty$ is a valid upper bound for $x_i$, $i=1,\ldots,n$.
Problem \eqref{eq:denoisingMIO} is a convex MIQO problem, which can be tackled using off-the-shelf MIO solvers. Estimation problems with a few hundred of variables can be comfortably solved to optimality using such solvers, e.g., see \cite{bertsimas2015or,cozad2014learning,gomez2018mixed,wilson2017alamo}. For high Signal-to-Noise Ratios (SNR), the estimators obtained from solving the exact $\ell_0$ problems indeed result in superior statistical performance when compared with the $\ell_1$ approximations \citep{bertsimas2016best}. For low SNR, however, the lack of \emph{shrinkage} may hamper the estimators obtained from optimal solutions of the $\ell_0$ problems \cite{hastie2017extended}; nonetheless, if necessary, shrinkage can be easily added to \eqref{eq:denoisingMIO} via conic quadratic regularizations terms \cite{mazumder2017subset}, resulting again in superior statistical performance over corresponding $\ell_1$-approximations. 
Unfortunately, current MIO solvers are unable to solve larger problems with thousands of variables.
 
\new{A recent research thrust aims to use MIO formulations and heuristics for sparse learning problems such as \eqref{eq:denoisingMIO} \cite{hazimeh2018fast,xie2018scalable}, which scale to larger instances than MIO solvers but may not provide dual bounds on the quality of the solutions found. Another research direction aims to design tailored exact methods for specialized regression problems \cite{atamturk2020safe,bertsimas2020sparse,hazimeh2020sparse,kimura2018minimization}, which perform substantially better than general-purpose MIO solvers for the problems their are designed to tackle, but may not generalize well to other learning problems; in particular, general constraints such as \eqref{eq:cardinality} are challenging to handle via tailored algorithms.}

Finally, we point out the relationship between the $\ell_1$-approximation \eqref{eq:lasso} and the MIO formulation \eqref{eq:denoisingMIO}. It can be verified easily that, \rev{if $C$ is defined by a $k$-sparsity constraint, then} there exists an optimal solution $z$ to the simple convex relaxation with big-$M$ constraint, where $z_i=\frac{x_i}{u}$ for all $i=1,\ldots,n$. Therefore, the constraint \eqref{eq:cardinality} reduces to $\|x\|_1\leq ku$, and we find that \eqref{eq:lasso} is in fact the natural convex relaxation of \eqref{eq:denoisingMIO} (for a suitable sparsity parameter). This relaxation is often weak and can be improved substantially.

\subsection{The perspective reformulation}
\label{sec:perspective}

A simple strengthening technique to improve the convex relaxation of \eqref{eq:denoisingMIO} is the \texttt{perspective reformulation} \cite{frangioni2006perspective}, which will be referred to as \persp in the remainder of the paper for brevity.
This reformulation technique can be applied to the estimation error terms in \eqref{eq:objective} as follows:
\begin{align}
 (y_i-x_i)^2\leq t \;&  \rev{\Leftrightarrow\;y_i^2-2y_ix_i+x_i^2\leq t}\notag\\
 &\to\; y_i^2-2y_ix_i+\frac{x_i^2}{z_i}\leq t. \label{eq:perspective}
\end{align}
The term $x_i^2/z_i$ is the closure of the perspective function of the quadratic function $x_i^2$, and is therefore convex, see p. 160 of \cite{hiriart2013convex}. Reformulation \eqref{eq:perspective} is in fact the best possible for \emph{separable} quadratic functions with indicator variables.
The perspective terms $\frac{x_i^2}{z_i}$ can be replaced with an auxiliary variable $s_i$ along with
rotated cone constraints $x_i^2\leq s_iz_i$ \cite{akturk2009strong,gunluk2010perspective}.
Therefore, \texttt{persp} relaxations can be easily solved with conic quadratic solvers and is by now a standard technique for mixed-integer quadratic optimization \cite{bonami2015mathematical,hijazi2012mixed,mahajan2017minotaur,wu2017quadratic}. Additionally, relationships between the \persp and the sparsity-inducing non-convex penalty functions \texttt{minimax concave penalty} \cite{zhang2010nearly} and \texttt{reverse Huber penalty} \cite{pilanci2015sparse} have recently been established \cite{dong2015regularization}. In the context of the signal estimation problem \eqref{eq:denoising}, the \persp yields the convex relaxation
\begin{align*}
\sum_{i=1}^ny_i^2+\min\;&\sum_{i=1}^n(-2y_ix_i+\frac{x^2_i}{z_i})+\lambda\sum_{\rev{\{i,j\}}\in A}(x_i-x_j)^2\\
(\texttt{persp.})  \ \ \ \ \ \text{s.t.}\;
&x_i\leq \|y\|_\infty z_i&i=1,\ldots,n\\
&\rev{z\in \bar C}, \
x\in \R_+^n, 
\end{align*}
\rev{where $\bar C$ is a valid convex relaxation of $C$, e.g., $\bar C=\text{conv}(C)$.}
The $\ell_1$-approximation model, as discussed in Section~\ref{sec:l1norm}, is the best convex relaxation that considers only the indicators for the $\ell_0$ terms. The \persp approximation is the best convex relaxation that exploits the $\ell_0$ indicator variables as well as the separable quadratic estimation error terms; thus, it is stronger than the $\ell_1$-approximation. However, \persp cannot be applied to non-separable quadratic smoothness terms $(x_i-x_j)^2$, as the function $x_i^2/z_i-2x_ix_j+x_j^2/z_j$ is non-convex due to the bilinear term. 

\subsection{Strong formulations for pairwise quadratic terms}
\label{sec:strong2var}

Recently, \citet{jeon2017quadratic} gave strong relaxations for the mixed-integer epigraphs of non-separable convex quadratic functions with two variables and indicator variables. \citet{atamturk2018strong} further strengthened the relaxations for quadratic functions of the form $(x_i-x_j)^2$ corresponding to the smoothness terms in \eqref{eq:denoisingMIO}.
Specifically, let 
$$ X^2=\left\{(z,x,s)\in \{0,1\}^2\times \R_+^3: (x_1-x_2)^2\leq s,\; x_i(1-z_i)=0, i=1,2\right\}$$ 
and define the function  $f:[0,1]^2\times \R_+^2\to \R_+$ as 
$$f(z,x)=\begin{cases}\frac{(x_1-x_2)^2}{z_1}&\text{if }x_1\geq x_2\\ \frac{(x_1-x_2)^2}{z_2}&\text{if }x_1\leq x_2.\end{cases}$$

\begin{proposition}[\citet{atamturk2018strong}] 
	\label{prop:MHullSimple}
	The function $f$ is convex and $\conv(X^2)=\left\{(z,x,s)\in [0,1]^2\times \R_+^3: f(z,x)\leq s\right\}.$
\end{proposition}
Using \persp and Proposition~\ref{prop:MHullSimple}, one obtains the stronger \pairwise convex relaxation of \eqref{eq:denoisingMIO} as
\begin{subequations}\label{eq:M-natOriginal}
\begin{align}
\sum_{i=1}^ny_i^2+\min\;&\sum_{i=1}^n\left(-2y_ix_i+\frac{x_i^2}{z_i}\right)+\lambda\sum_{\rev{\{i,j\}}\in A}f(z_i,z_j,x_i,x_j)\\
(\pairwise) \ \ \ \text{ s.t.}\;
&x_i\leq \|y\|_\infty z_i, \quad\quad\quad i=1,\ldots,n\\
&\rev{z\in \bar C}, \ x\in \R_+^n.
\end{align}
\end{subequations}
Note that $f$ is not differentiable everywhere and it is defined by pieces. Therefore, it cannot be used directly with most convex optimization solvers. \citet{atamturk2018strong} implement \eqref{eq:M-natOriginal} using linear outer approximations of function $f$: the resulting method performs adequately for instances with $n\leq 400$, but was ineffective in instances with $n\geq 1,000$ as strong linear outer approximations require the addition of a large number of constraints. 
Moreover, as Example~\ref{ex:M2var} below shows, formulation \eqref{eq:M-natOriginal} can be further improved even for $n=2$. 
\begin{example}
	\label{ex:M2var}
	Consider the signal estimation problem \eqref{eq:L0lagrange} with $n=2$
	\begin{subequations}\label{eq:exampleM2D}
		\begin{align} \min\;& (0.4-x_1)^2 + (1-x_2)^2+0.5(x_1-x_2)^2+0.5\left(z_1+z_2\right)\\
		\text{s.t.}\;&x_i\leq z_i,\ i=1,2\\
		&z\in \{0,1\}^2, x\in \R_+^2.\end{align}
	\end{subequations}
The optimal solution of \eqref{eq:exampleM2D} is $(z_1^*,z_2^*,x_1^*,x_2^*)=(0.00,1.00,0.00,0.67)$. \rev{On the other hand,}
	optimal solutions of the convex relaxations of \eqref{eq:exampleM2D} are: 
	\begin{description}
		\item[\rev{$\ell_1$-approx}] Obtained by replacing $z\in \{0,1\}^2$ with $z\in [0,1]^2$. The corresponding optimal solution is $(z_\ell,x_\ell)=(0.30,0.60,0.30,0.60)$, and we find that $\left\|(z^*,x^*)-(z_\ell,x_\ell)\right\|_2=0.59$.
		\item[\texttt{persp}] The optimal solution is $(z_p,x_p)=(0.00,0.82,0.00,0.59)$, and $\left\|(z^*,x^*)-(z_p,x_p)\right\|_2=0.19$. 
		\item[\texttt{pairwise}] The optimal solution is $(z_q,x_q)=(0.11,1.00,0.08,0.69)$, and $\left\|(z^*,x^*)-(z_q,x_q)\right\|_2=0.14$. 
		\end{description} 
	Although \persp and \pairwise substantially improve upon the \rev{$\ell_1$-relaxation}, the resulting solutions are \rev{still} not integral in $z$. \rev{We will give the convex hull of \eqref{eq:exampleM2D}} in the next section.  \hfill $\qed$

\ignore{ 
	The optimal solutions and objective values of the convex relaxations of \eqref{eq:exampleM2D} are as follows: \texttt{lasso}, obtained by replacing $z\in \{0,1\}^2$ with $z\in [0,1]^2$, has optimal solution $(z_1,z_2,x_1,x_2)=(0.30,0.60,0.30,0.60)$, with optimal objective value $\zeta_\lasso=0.665$; \persp has optimal solution $(z,x)=(0.00,0.82,0.00,0.59)$, with objective value $\zeta_\persp=0.988$; and convex relaxation \eqref{eq:M-natOriginal} has optimal solution $(z,x)=(0.11,1.00,0.08,0.69)$, with objective value $\zeta_\eqref{eq:M-natOriginal}=0.991$. The optimal solution of \eqref{eq:exampleM2D} is $(z,x)=(0.00,1.00,0.00,0.67)$, with optimal objective value $\zeta^*=0.993$. Thus, although \persp and \eqref{eq:M-natOriginal} substantially improve upon the \lasso relaxation, the resulting solutions are not integral in $z$ and there is a gap between the optimal values of such relaxations and $\zeta^*$. 
	\hfill $\qed$ 
}
\end{example}

In this paper, we show how to further improve the \pairwise formulation  to obtain a stronger relaxation of \eqref{eq:denoisingMIO}. Additionally, we show how to implement the \rev{relaxations} derived in the paper in a conic quadratic optimization framework. 
Therefore, the proposed convex relaxations benefit from a 
growing literature on conic quadratic optimization, e.g., see \cite{alizadeh2003second,AG:sub-cqmip,cmir-ipco,lobo1998applications,nemirovski2008interior}, can be implemented with off-the-shelf solvers, and scale to large instances.

\section{Strong convex formulations for signal estimation}
\label{sec:Mmatrix}

In the \pairwise formulation each single- and two-variable quadratic term is strengthened independently and, consequently, the formulation fails to fully exploit the relationships between different pairs of variables. Observe that problem \eqref{eq:denoisingMIO} can be stated as
\rev{ 
\begin{subequations}\label{eq:optMmatrix}
	\begin{align} \|y\|_2^2+ \min\; &-2y'x+x'Qx\\
	\text{s.t.}\;& x_i(1-z_i)=0, i=1\ldots,n,\\
	&z\in C, x\in \R_+^n\end{align}
\end{subequations}}
\noindent
where, for $i \neq j$, $Q_{ij}=-\lambda$ if $\rev{\{i,j\}}\in A$ and  $Q_{ij}=0$ otherwise, and $Q_{ii}=1+\lambda |A_i|$ where $A_i=\left\{j: \rev{\{i,j\}}\in A\right\}$. In particular, $Q$ is a symmetric M-matrix, i.e., $Q_{ij}\leq 0$ 
\rev{for $i \neq j$} and $Q\succeq 0$. In this section we derive convex relaxations of \eqref{eq:denoisingMIO} that better exploit the M-matrix structure. We briefly review properties of M-matrices and refer the reader to \cite{berman1994nonnegative,gao1992criteria,plemmons1977m,varga1976recurring} and the references therein for an in-depth discussion on M-matrices. 
\begin{proposition}[\citet{plemmons1977m}, characterization 37]
An M-matrix is generalized diagonally dominant, i.e., there exists a positive diagonal matrix $D$ such that $DQ$ is (weakly) diagonally dominant.
\end{proposition}

Generalized diagonally dominant matrices are also called scaled diagonally dominant matrices in the literature.
\begin{proposition}[\citet{boman2005factor}]
	\label{prop:decomposable}
	A matrix $Q$ is generalized diagonally dominant iff it has factor width at most two, i.e., there exists a real matrix $V_{n\times m}$ such that $Q=VV^\top$ and each column of $V$ contains at most two non-zeros. 
\end{proposition}
Proposition~\ref{prop:decomposable} implies that if $Q$ is an M-matrix, then the quadratic function $x'Qx$ can be written as a sum of quadratic functions of at most two variables each, i.e., $x'Qx=\sum_{j=1}^m\left(\sum_{i=1}^nV_{ij}x_i\right)^2$ where for any $j$ at most two entries $V_{ij}$ are non-zero. Therefore, to derive stronger formulations for \eqref{eq:optMmatrix}, we first study the mixed-integer epigraphs of \emph{parametric} pairwise quadratic functions with indicators. 

\subsection{Convexification of the parametric pairwise terms }
\label{sec:M2General}

Consider the mixed-integer epigraph of a parametric pairwise quadratic term (with parameters $d_1, d_2$)
\begin{align*}Z^2 \! = \! \Big\{(z,x,s)\in \{0,1\}^2 \! \times \! \R_+^3 \! :& \ d_1x_1^2-2x_1x_2+d_2x_2^2\leq s,\\
 & \ x_i(1-z_i)\rev{=0},\ i=1,2\Big\},\end{align*}
where $d_1d_2\geq 1$ and $d_1,d_2 > 0$, 
which is the necessary and sufficient condition for convexity of the function $d_1x_1^2-2x_1x_2+d_2x_2^2$. One may, \rev{without loss of generality,} assume the cross-product coefficient equals $-2$, as otherwise the continuous variables and coefficients can be scaled. Clearly, if $d_1=d_2=1$, then $Z^2$ reduces to $X^2$.

Consider the \rev{two} decompositions of the two-variable quadratic function in the definition of $Z^2$ given by
 \begin{align*}
 d_1x_1^2-2x_1x_2+d_2x_2^2&=d_1\left(x_1-\frac{x_2}{d_1}\right)^2+x_2^2\left(d_2-\frac{1}{d_1}\right)\\
 &=d_2\left(\frac{x_1}{d_2}-x_2\right)^2+x_1^2\left(d_1-\frac{1}{d_2}\right).
 \end{align*}
 Intuitively, the decompositions above are obtained by extracting a term $\delta_ix_i^2$ from the quadratic function such that $\delta_i$ is as large as possible and the remainder quadratic term is still convex. Then, applying \persp and Proposition~\ref{prop:MHullSimple} to the separable and pairwise quadratic terms, respectively, one obtains two valid inequalities for $Z^2$:
 \begin{align}
 d_1f(z_1,z_2,x_1,\frac{x_2}{d_1})+\frac{x_2^2}{z_2}\left(d_2-\frac{1}{d_1}\right)&\leq s\label{eq:negativeDecomp2}\\
 d_2f(z_1,z_2,\frac{x_1}{d_2},x_2)+\frac{x_1^2}{z_1}\left(d_1-\frac{1}{d_2}\right)&\leq s\label{eq:negativeDecomp1}.
 \end{align}
 \rev{Clearly, there are infinitely many such decompositions depending on the values of $\delta_i$, $i=1,2$. Surprisingly,}
 Theorem~\ref{theo:MHullGeneral} below shows that inequalities \eqref{eq:negativeDecomp2}--\eqref{eq:negativeDecomp1} along with the bound constraints are sufficient to describe $\conv(Z^2)$.
 \begin{theorem}
 	\label{theo:MHullGeneral}
 	$\conv(Z^2)=\left\{(z,x,s)\in [0,1]^2\times \R_+^3:\eqref{eq:negativeDecomp2}-\eqref{eq:negativeDecomp1}\right\}.$
 \end{theorem}
 \begin{proof}
 	Consider the mixed-integer optimization problem
 	\begin{equation}
 	\label{eq:discreteNegative1}
 	\min_{(z,x,s)\in Z^2}a_1z_1+a_2z_2+b_1x_1+b_2x_2+\lambda s
 	\end{equation}
 	and the corresponding convex optimization
 	\begin{subequations}
 		\label{eq:convexNegative1}
 		\begin{align}
 		\min\;&a_1z_1+a_2z_2+b_1x_1+b_2x_2+\lambda s\\
 		\text{s.t.}\;&d_1f(z_1,z_2,x_1,\frac{x_2}{d_1})+\frac{x_2^2}{z_2}\left(d_2-\frac{1}{d_1}\right)\leq s \label{eq:convexNegEq}\\
 		&d_2f(z_1,z_2,\frac{x_1}{d_2},x_2)+\frac{x_1^2}{z_1}\left(d_1-\frac{1}{d_2}\right)\leq s\label{eq:convexNegDropped}\\
 		&z\in [0,1]^2,\; x\in \R_+^2,\; s\in \R_+.
 		\end{align}
 	\end{subequations}
To prove the result it suffices to show that, for any value of $(a,b,\lambda)$, either \eqref{eq:discreteNegative1} and \eqref{eq:convexNegative1} are both unbounded, or that \eqref{eq:convexNegative1} has an optimal solution that is also optimal for \eqref{eq:discreteNegative1}. We assume, without loss of generality, that $d_1d_2>1$ (if $d_1d_2=1$, the result follows from Proposition~\ref{prop:MHullSimple} by scaling), $\lambda>0$ (if $\lambda <0$, both problems are unbounded by letting $s\to\infty$, and if $\lambda=0$, problem \eqref{eq:convexNegative1} reduces to linear optimization over a integral polytope and optimal solutions are integral in $z$), and $\lambda =1$ (by scaling). Moreover, since $d_1d_2>1$, there exists an optimal solution for both \eqref{eq:discreteNegative1} and \eqref{eq:convexNegative1}.

Let $(z^*,x^*, s^*)$ be an optimal solution of \eqref{eq:convexNegative1}; we show how to construct from $(z^*,x^*, s^*)$ a feasible solution for \eqref{eq:discreteNegative1} with same objective value, thus optimal for both problems. Observe that for $\gamma\geq 0$, $f(\gamma z_1,\gamma z_2,\gamma x_1,\gamma x_2)=\gamma f(z_1,z_2,x_1,x_2)$. Thus, if $z_1^*,z_2^*<1$, then  $(\gamma z^*,\gamma x^*, \gamma s^*)$ is also feasible for \eqref{eq:convexNegative1} 
with objective value $\gamma\left(a_1z_1^*+a_2z_2^*+b_1x_1^*+b_2x_2^*+s^*\right)$. In particular, either there exists an (integral) optimal solution with $z^*=x^*=0$ by setting $\gamma=0$, or there exists an optimal solution with one of the $z$ variables equal to one  by increasing $\gamma$. Thus, assume without loss of generality that $z_1^*=1$. Now consider the optimization problem
 	\begin{subequations}
 		\label{eq:convexNegative2}
 		\begin{align}
 		\min\;&a_2z_2+b_1x_1+b_2x_2+d_1f(1,z_2,x_1,\frac{x_2}{d_1})+\frac{x_2^2}{z_2}\left(d_2-\frac{1}{d_1}\right)\\
 		&z_2\in [0,1],\; x\in \R_+^2,
 		\end{align}
 	\end{subequations}
 	obtained from \eqref{eq:convexNegative1} by fixing $z_1=1$, dropping constraint \eqref{eq:convexNegDropped}, and eliminating variable $s$ since \eqref{eq:convexNegEq} holds at equality in optimal solutions. 
 	An integer optimal solution for \eqref{eq:convexNegative2} is also optimal for \eqref{eq:discreteNegative1} and \eqref{eq:convexNegative1}. Let $(\hat z, \hat x)$ be an optimal solution for \eqref{eq:convexNegative2}, and consider the two cases: 
 	
    \noindent \textit{Case 1}:
 	 $\hat x_1\leq \hat x_2/d_1$: If $0<\hat z_2<1$, then the point $(\gamma \hat z_2, \gamma \hat x_1, 
 	\gamma \hat x_2)$ 
 	with $0 \le \gamma \hat z_2 \le 1 $ is feasible for \eqref{eq:convexNegative2} 
 	with objective value $$\gamma\left(a_2\hat z_2+b_1\hat x_1+b_2\hat x_2+d_1f(1,\hat z_2,\hat x_1,\frac{\hat x_2}{d_1})+\frac{\hat x_2^2}{\hat z_2}\left(d_2-\frac{1}{d_1}\right)\right).$$
 		Therefore, there exists an optimal solution where $\hat z_2\in \{0,1\}$. \hfill $\qed$

     \noindent \textit{Case 2}:	   
     $\hat x_1> \hat x_2/d_1$: In this case, $(\hat z_2,\hat x_1, \hat x_2)$ is an optimal solution of 
 		\begin{subequations}\label{eq:intermediate1}
 			\begin{align}
 			\min\;&a_2z_2+b_1x_1+b_2x_2+d_1\left(x_1-\frac{x_2}{d_1}\right)^2+\frac{x_2^2}{z_2}\left(d_2-\frac{1}{d_1}\right)\\
 			&z_2\in [0,1],\; x\in \R_+^2.
 			\end{align}
 		\end{subequations}
 	The condition  $\hat x_1> \hat x_2/d_1$ implies that $\hat x_1>0$, thus the optimal value of $x_1$ can be found by taking derivatives and setting to $0$. We find
 	$$\hat x_1=-\frac{b_1}{2d_1}+\frac{x_2}{d_1} \cdot $$
 	Replacing $x_1$ with his optimal value in \eqref{eq:intermediate1} and removing constant terms, we find that \eqref{eq:intermediate1} is equivalent to 
 		\begin{subequations}\label{eq:intermediate2}
 		\begin{align}
 		\min\;&a_2z_2+\left(\frac{b_1}{d_1}+b_2\right)x_2+\frac{x_2^2}{z_2}\left(d_2-\frac{1}{d_1}\right)\\
 		&z_2\in [0,1],\; x_2\in \R_+.
 		\end{align}
 	\end{subequations}
 If $0<\hat z_2<1$, then the point $(\gamma \hat z_2, \gamma \hat x_2)$ with $0 \le \gamma \hat z_2 \le 1 $ is feasible for \eqref{eq:intermediate2} with objective value $$\gamma\left(a_2\hat z_2+\left(\frac{b_1}{d_1}+b_2\right)\hat x_2+\frac{\hat x_2^2}{\hat z_2}\left(d_2-\frac{1}{d_1}\right)\right).$$
 Therefore, there exists an optimal solution where $\hat z_2\in \{0,1\}$. \hfill $\qed$
 		
 In both cases we find an optimal solution with $z_2\in \{0,1\}$. Thus, problem \eqref{eq:convexNegative1} has an optimal solution integral in both $z_1$ and $z_2$, which is also optimal for \eqref{eq:discreteNegative1}.
 \end{proof}
\setcounter{example}{0}
\begin{example}[continued]
The relaxation of \eqref{eq:exampleM2D} with only inequality \eqref{eq:negativeDecomp1}:
\begin{align*}1.16+\min\;& -0.8x_1 -2x_2+0.5\left(z_1+z_2\right)+0.5s\\
\text{s.t.}\;
&3f(z_1,z_2,\frac{x_1}{3},x_2)+\frac{x_1^2}{z_1}\left(3-\frac{1}{3}\right)\leq s\\
&z\in [0,1]^2, x\in \R_+^2,\end{align*}
is sufficient to obtain the integral optimal solution. Note that the big-$M$ constraints $x_i\leq z_i$ are not needed. \hfill $\qed$
\end{example}

Given $d_1,d_2\in \R_+$, define the function $g:[0,1]^2\times \R_+^2\to \R_+$ as \begin{align}
\label{eq:hConvex}g(z_1,z_2,x_1,x_2;d_1,d_2)=\max\Bigg\{&d_1f(z_1,z_2,x_1,\frac{x_2}{d_1})+\frac{x_2^2}{z_2}\left(d_2-\frac{1}{d_1}\right),\notag\\
&d_2f(z_1,z_2,\frac{x_1}{d_2},x_2)+\frac{x_1^2}{z_1}\left(d_1-\frac{1}{d_2}\right)\Bigg\} \cdot
\end{align}
For any $d_1, d_2>0$ with $d_1d_2\geq 1$, function $g$ is the point-wise maximum of two convex functions and is therefore convex. Using the convex function $g$,
Theorem~\ref{theo:MHullGeneral} can be restated as $$\conv(Z^2)=\left\{(z,x,s)\in [0,1]^2\times \R_+^3: g(z_1,z_2,x_1,x_2;d_1,d_2)\leq s\right\}.$$
Finally, it is easy to \rev{verify} that if $z_1\geq z_2$, then the maximum in \eqref{eq:hConvex} corresponds to the first term; if $z_1\leq z_2$, the maximum corresponds to the second term. Thus, an explicit expression of $g$ is 
$$g(z,x;d)=\begin{cases}\frac{d_1x_1^2-2x_1x_2+x_2^2/d_1}{z_1}+\frac{x_2^2}{z_2}\left(d_2-\frac{1}{d_1}\right)&\text{if }z_1\geq z_2\text{ and }d_1x_1\geq x_2\\
\frac{d_1x_1^2-2x_1x_2+d_2x_2^2}{z_2}&\text{if }z_1\geq z_2\text{ and }d_1x_1\leq x_2\\
\frac{d_1x_1^2-2x_1x_2+d_2x_2^2}{z_1}&\text{if }z_1\leq z_2\text{ and }x_1\geq d_2x_2\\
\frac{x_1^2/d_2-2x_1x_2+d_2x_2^2}{z_2}+\frac{x_1^2}{z_1}\left(d_1-\frac{1}{d_2}\right)&\text{if }z_1\le z_2\text{ and }x_1\leq d_2x_2.\end{cases}$$ 

\subsection{Convex relaxations for general M-matrices} 
\label{sec:Mseparation}
Consider the set 
$$Z^n=\left\{(z,x,t)\in \{0,1\}^n\times \R_+^{n+1}: x'Qx\leq t,\; x_i(1-z_i)\rev{=0}, i=1,\ldots,n\right\},$$
where $Q$ is an M-matrix. In this section,
we will show how the convex hull descriptions for $Z^2$ can be used to construct strong convex relaxations for $Z^n$. We start with the following motivating example.
\begin{example}
	\label{ex:Mdecomp}
Consider the signal estimation in regularized form with $n=3$, $(y_1,y_2,y_3)=(0.3,0.7,1.0)$, $\lambda=1$ and $\mu=0.5$,
\begin{subequations}\label{eq:exampleM3D}
\begin{align}
\zeta=\rev{1.58}+\min\;& -0.6x_1-1.4x_2-2.0x_3+t+0.5\left(z_1+z_2+z_3\right)\\
\text{s.t.}\;&x_1^2+x_2^2+x_3^2+(x_1-x_2)^2+(x_2-x_3)^2\leq t\label{eq:quadratic}\\
&x_i\leq z_i,\quad\quad\quad i=1,2,3\\
&z\in \{0,1\}^3,\; x\in \R_+^3.
\end{align}
\end{subequations}
The optimal solution of \eqref{eq:exampleM3D} is $(z^*,x^*)=(0.00,1.00,1.00,0.00,0.48,0.74)$ with objective value $\zeta^*=1.504$. 
The optimal solutions and the corresponding objective values of the convex relaxations of \eqref{eq:exampleM3D} are as follows:
\begin{description} 
	\item[\texttt{$\ell_1$-approx}] The opt. solution is $(z_\ell,x_\ell)=(0.24,0.43,0.59,0.24,0.43,0.59)$ with value $\zeta_\lasso=0.936$, and $\left\|(z^*,x^*)-(z_\ell,x_\ell)\right\|_2=0.80$.
	\item[\texttt{persp}] The opt. solution is $(z_p,x_p)=(0.00,0.40,0.82,0.00,0.29,0.58)$ with value $\zeta_\persp=1.413$, and $\left\|(z^*,x^*)-(z_p,x_p)\right\|_2=0.67$.
	\item [\texttt{pairwise}] The opt. solution $(z_{q},x_{q})=(0.18,0.74,1.00,0.13,0.43,0.71)$ with value $\zeta_\pairwise=1.488$, and $\left\|(z^*,x^*)-(z_{q},x_{q})\right\|_2=0.35$.
	\item[\texttt{decomp.1}]The quadratic constraint \eqref{eq:quadratic} can be decomposed and strengthened as follows:
\begin{align*}
\left(2x_1^2-2x_1x_2+x_2^2\right)+\left(2x_2^2-2x_2x_3+2x_3^2\right)&\leq t\\
\to\;g(z_1,z_2,x_1,x_2;2,1)+g(z_2,z_3,x_2,x_3;2,2)&\leq t;
\end{align*}
leading to solution is $(z_{d},x_{d})=(0.17,1.00,0.93,0.12,0.53,0.73)$ with value $\zeta_{\texttt{decomp.1}}=1.495$, and $\left\|(z^*,x^*)-(z_{d},x_{d})\right\|_2=0.23$. 
\item[\texttt{decomp.2}]Alternatively, constraint \eqref{eq:quadratic} can also be formulated as $g(z_1,z_2,x_1,x_2;2,2)+g(z_2,z_3,x_2,x_3;1,2)\leq t$, and the resulting convex relaxation has solution $(z^*,x^*)=(0.00,1.00,1.00,0.00,0.48,0.74)$, corresponding to the optimal solution of \eqref{eq:exampleM3D}.
\hfill $\qed$
\end{description}
\end{example}
As Example~\ref{ex:Mdecomp} shows, strong convex relaxations of $Z^n$ can be obtained by decomposing $x'Qx$ into sums of two-variable quadratic terms (as $Q$ is an M-matrix) and convexifying each term. However, such a decomposition is not unique and the strength of the relaxation depends on the decomposition chosen. We now discuss how to \rev{optimally} decompose the matrix $Q$ to derive the strongest lower bound possible \rev{for a fixed value of $(z,x,t)$. Then, we show how this decomposition procedure can be embedded in a cutting surface algorithm to obtain a strong convex relaxation of \eqref{eq:optMmatrix}.}   

Consider the \emph{separation problem}: given a point $(z,x,t)\in [0,1]^n\times \R_+^{n+1}$, find a decomposition of $Q$ such that, after strengthening each two-variable term, results in a most violated inequality, which is formulated as follows:
\begin{subequations}\label{eq:Mseparation} \small
\begin{align}
\theta(z,x)=\max_{d}\;&\sum_{i=1}^n\sum_{j=i+1}^n|Q_{ij}|g(z_i,z_j,x_i,x_j; d_{ij}^i,d_{ij}^j)\label{eq:MsepObj}\\
\text{s.t.}\;& \sum_{j<i}|Q_{ji}|d_{ji}^i+\sum_{j>i}|Q_{ij}|d_{ij}^i =Q_{ii}\quad &\forall i=1,\ldots,n \label{eq:MsepCorrect}\\
&d_{ij}^id_{ij}^j\geq 1,\; d_{ij}^i\geq 0,\; d_{ij}^j\geq 0\quad &\forall i<j. \label{eq:MsepConvex}
\end{align}
\end{subequations} \normalsize
Observe that the variables of the \rev{separation} problem \eqref{eq:Mseparation} are the parameters $d$, and the variables of the estimation problem $(z,x)$ are fixed in the separation problem. 
In formulation \eqref{eq:Mseparation} for each (negative) entry $Q_{ij}$, $i<j$, there is a two-variable quadratic term of the form $|Q_{ij}|\left(d_{ij}^ix_i^2-2x_ix_j+d_{ij}^jx_j^2\right)$; after convexifying each such term, one obtains the objective \eqref{eq:MsepObj}. Constraints \eqref{eq:MsepCorrect} ensure that the decomposition indeed corresponds to the original matrix $Q$ by ensuring that the diagonal elements coincide, and constraints \eqref{eq:MsepConvex} ensure that each quadratic term is convex. From Proposition~\ref{prop:decomposable}, problem \eqref{eq:Mseparation} is feasible for any M-matrix $Q$. 

For any feasible value of $d$, the objective \eqref{eq:MsepObj} is convex in $(z,x)$; thus the function $\theta:[0,1]^n\times \R_+^n\to \R_+$ defined in \eqref{eq:Mseparation} is a supremum of convex functions and is convex itself. Moreover, the constraints \eqref{eq:MsepCorrect} and \eqref{eq:MsepConvex} are linear or rotated cone constraints, thus, are convex in $d$. As we now show, the objective function \eqref{eq:MsepObj} is concave in $d$, thus \eqref{eq:Mseparation} is a convex optimization.

Index the variables such that $z_1\geq z_2\geq\ldots\geq z_n$. Then, each term in the objective \eqref{eq:MsepObj} reduces to
\small
\begin{multline}
g(z_i,z_j,x_i,x_j; d_{ij}^i,d_{ij}^j)=\begin{cases}\frac{d_{ij}^ix_i^2-2x_ix_j+x_j^2/d_{ij}^i}{z_i}+\frac{x_j^2}{z_j}\left(d_{ij}^j-\frac{1}{d_{ij}^i}\right)&\text{if }d_{ij}^ix_i\geq x_j\\
\frac{d_{ij}^ix_i^2-2x_ix_j+d_{ij}^jx_j^2}{z_j}&\text{if }d_{ij}^ix_i\leq x_j
\end{cases}\notag\\
=d_{ij}^i\frac{x_i^2}{z_i}\!+\!d_{ij}^j\frac{x_j^2}{z_j} \!+ \! \! \begin{cases}\frac{-2x_ix_j}{z_i}-\frac{x_j^2}{d_{ij}^i}\! \left(\frac{1}{z_j}\!-\!\frac{1}{z_i}\right)\! \! & \! \! \text{if }d_{ij}^ix_i\geq x_j\\
\frac{-2x_ix_j}{z_j}+d_{ij}^ix_i^2 \!\left(\frac{1}{z_j}\!-\!\frac{1}{z_i}\right) \! \! & \! \! \text{if }d_{ij}^ix_i\leq x_j.\end{cases}\label{eq:hSepForm}
\end{multline}\normalsize
Thus, $g(z,x;d)$ is separable in $d_{ij}^i$ and $d_{ij}^j$, is linear in $d_{ij}^j$; and, it is linear in $d_{ij}^i$ for $d_{ij}^i \le x_j/x_i$, and concave for $d_{ij}^i \ge  x_j/x_i$. Moreover, it is easily shown that it is continuous and differentiable (i.e., the derivatives of both pieces of $g$ with respect to $d_{ij}^i$ coincide if $d_{ij}^ix_i=x_j$). Therefore, the separation problem \eqref{eq:Mseparation} can be solved in polynomial time by first sorting the variables $z_i$ and then by solving a convex optimization problem. 

\rev{
The separation procedure can be embedded in an algorithm that iteratively constructs stronger relaxations of problem \eqref{eq:optMmatrix}.

\noindent\texttt{Simple cutting surface algorithm}:
\begin{enumerate}
	\item[1.] Solve a valid convex relaxation. 
	\item[2.] Solve separation problem \eqref{eq:Mseparation} using a convex optimization method.
	\item[3.] Add the inequality obtained from solving the separation problem to the formulation, strengthening the relaxation, and go to step~1.
\end{enumerate}
Below, we illustrate the \texttt{simple cutting surface algorithm}.
}

\setcounter{example}{1}
\begin{example}[Continued]
	Consider the \persp relaxation
	\begin{subequations}\label{eq:formulationExample}
	\begin{align}
	\zeta_1=\rev{1.58}+\min\;& -0.6x_1-1.4x_2-2.0x_3+t+0.5\left(z_1+z_2+z_3\right)\\
	\text{s.t.}\;&\frac{x_1^2}{z_1}+\frac{x_2^2}{z_2}+\frac{x_3^2}{z_3}+(x_1-x_2)^2+(x_2-x_3)^2\leq t\\
	&x_i\leq z_i,\quad i=1,2,3\\
	&z\in [0,1]^3,\; x\in \R_+^3.
	\end{align}
	\end{subequations}
with optimal solution $(z, x)_1=(0.00,0.40,0.82,0.00,0.29,0.58)$ with $\zeta_1=1.413$ and $\left\|(z^*,x^*)-(z,x)_1\right\|_2=0.67$. This relaxation can be improved by solving the separation problem \eqref{eq:Mseparation} at $( z,x)_1$ to obtain the optimal parameters $d_{12}^1=2.00$, $d_{12}^2=0.51$, $d_{23}^2=2.49$ and $d_{23}^3=2.00$, leading to the decomposition and the constraint
$$g(z_1,z_2,x_1,x_2;2.00,0.51)+g(z_2,z_3,x_2,x_3;2.49,2.00)\leq t.$$
Adding this constraint to \eqref{eq:formulationExample} and resolving gives the improved solution 
$(z, x)_2=(0.15,0.70,1.00,0.12,0.43,0.71)$.
This process can be repeated iteratively, resulting in the sequence of solutions
\begin{description}
	\item[\texttt{iter.2}] $(z, x)_2=(0.15,0.70,1.00,0.12,0.43,0.71)$ with $\zeta_2=1.452$ and $\left\|(z^*,x^*)-(z,x)_2\right\|_2=0.36$. The corresponding separation problem has solution $(d_{12}^1,d_{12}^2,d_{23}^2,d_{23}^3)=(2,1.06,1.94,2)$.
	\item[\texttt{iter.3}] $(z, x)_3=(0.14,1.00,1.00,0.10,0.52,0.75)$ with $\zeta_3=1.499$ and $\left\|(z^*,x^*)-(z,x)_3\right\|_2=0.18$. The corresponding separation problem has solution $(d_{12}^1,d_{12}^2,d_{23}^2,d_{23}^3)=(2,2.5,0.5,2)$.
	\item[\texttt{iter.4}] $(z, x)_4=(0.00,1.00,1.00,0.00,0.48,0.74)$ with $\zeta_3=1.504$. The solution is integral and optimal for \eqref{eq:exampleM3D}.\hfill $\qed$
\end{description}

\end{example}

The iterative separation procedure outlined above ensures that $(z,x,t)$ satisfies the convex relaxation
$$\Theta=\left\{(z,x,t)\in [0,1]^n\times \R_+^{n+1}:\theta(z,x)\leq t\right\}$$ of $Z^n$ that dominates the \rev{\texttt{$\ell_1$-approx}}, \texttt{persp}, and \texttt{pairwise} \rev{and gives the strong relaxation of problem \eqref{eq:optMmatrix}, based on the optimal \texttt{decomposition} of matrix $Q$, given by 
	\small\begin{align*}
(\texttt{decomp})  \ \ \ \ \	\|y\|_2^2+ \! \! \min_{(z,x)\in [0,1]^n\times \R_+^n} \! \! \! -2y'x+\theta(z,x)\text{:}\ z\in \bar C,\; x_i\leq \|y\|_\infty z_i,\ i=1\ldots,n.\end{align*}\normalsize
In Section~\ref{sec:implementation} we discuss the efficient implementation of \decomp in a conic quadratic optimization framework.
}

\section{Conic quadratic representation and Lagrangian decomposition}
\label{sec:implementation}

\rev{Relaxation~\decomp simultaneously exploits sparsity, fitness and smoothness terms in \eqref{eq:denoising}
	and, therefore, dominates all of the relaxations discussed in Section~\ref{sec:background}.
However, the convex functions $f$ and $g$ can be pathological, as they are defined by pieces and are not differentiable everywhere. Handling function $\theta$ is challenging as it is non-differentiable, but also it is not given in closed form and requires solving optimization problem \eqref{eq:Mseparation} to evaluate.

In this section, we first show how to tackle \decomp effectively by formulating it as a conic quadratic optimization problem in an extended space. We then give a tailored Lagrangian decomposition method, which is amenable to parallel computing and highly scalable.}

\rev{
	\subsection{Extended formulations} 

The \texttt{simple cutting surface algorithm} to solve \texttt{decomp}, illustrated in Example~\ref{ex:Mdecomp}, is computationally cumbersome since: \emph{(i)} the separation problem (step 2) requires solving a constrained convex optimization problem; \emph{(ii)} each cut added (step 3) is dense (and thus problematic for optimization software); \emph{(iii)} a single cut is generated at each iteration; consequently, the method may require many iterations to converge.

\new{In this section, we show how to address these shortcomings with a conic quadratic extended formulation with auxiliary variables. In particular:
\begin{itemize}
	\item[\emph{(i)}] The separation problem can be solved in closed form (Proposition~\ref{prop:optimalD})-- eliminating the need to solve auxiliary optimization problems.
	\item[\emph{(ii)}] Cuts \eqref{eq:cut} can be added as inequalities with at most four variables. Most conic quadratic optimization solvers are designed to exploit sparsity to improve performance and numerical stability.
	\item[\emph{(iii)}] The method may add up to $\mathcal{O}(n^2)$ cuts per round, decreasing the total number of rounds and re-optimizations. More importantly, simple cuts (e.g., sparse or linear cuts) in this extended space may translate into highly nonlinear cuts when projected into the original space of variables, often resulting in additional strength. Indeed, MIO often relies on extended formulations as such formulations lead to substaintial improvements compared to working in the original space of variables \cite{A:robobj,atamturk2010conic,tawarmalani2005polyhedral,vielma2017extended}.
\end{itemize}
The proposed extended formulation leads to a method at least two orders-of-magnitude faster than the \texttt{simple cutting surface algorithm.}}

	
	Define additional variables $\Gamma\in \R^{n\times n}$ such that $\Gamma_{ij}=\Gamma_{ji}$; intuitively, variable $\Gamma_{ij}$ represents the product $x_ix_j$. Given an M-matrix $Q$, consider the convex optimization problem
	\small \begin{subequations}\label{eq:relaxF}
		\begin{align}
		\min_{(z,x,\Gamma)}\;&\|y\|_2^2-2y'x+\langle \Gamma, Q\rangle\label{eq:relaxF_obj}\\
		\text{s.t.}\;
		& \Gamma_{ii}z_i\geq x_i^2&\hspace{-6cm}\forall i=1,\ldots,n\hfill \label{eq:relaxF_persp}\\
		& 0\geq \max_{d_{ij}> 0}d_{ij}f(z_i,z_j,x_i,\frac{x_j}{d_{ij}})-\left(d_{ij} \Gamma_{ii}-2\Gamma_{ij}+\frac{1}{d_{ij}}\Gamma_{jj}\right)&\forall i<j\hfill\label{eq:relaxF_cuts}\\
		& 0\leq x_i\leq \|y\|_\infty z_i&\hspace{-3cm}i=1,\ldots,n\hfill\\
		&z\in \bar C,\; x\in \R_+^n,\; \Gamma\in \R^{n\times n}.
		\end{align}
	\end{subequations}
	\normalsize
	
We will show in this section that problem \eqref{eq:relaxF} is equivalent to \decomp under mild conditions, and can be implemented efficiently via conic quadratic optimization. In order to prove this result, we introduce the auxiliary formulation:
	\small \begin{subequations}\label{eq:relaxG}
		\begin{align}
		\min_{(z,x,\Gamma)}\;&\|y\|_2^2-2y'x+\langle \Gamma, Q\rangle\label{eq:relaxG_obj}\\
		\text{s.t.}\;
		& 0\geq \max_{\substack{d_{ij}^id_{ij}^j\geq 1\\d_{ij}^i,d_{ij}^j\geq 0}}g(z_i,z_j,x_i,x_j; d_{ij}^i,d_{ij}^j)-\left(d_{ij}^i \Gamma_{ii}-2\Gamma_{ij}+d_{ij}^j\Gamma_{jj}\right)&\hspace{-0.1cm}\forall i<j\hfill\label{eq:relaxG_pair}\\
		& 0\leq x_i\leq \|y\|_\infty z_i&\hspace{-3cm}i=1,\ldots,n\quad\hfill\\
		&z\in \bar C,\; x\in \R_+^n,\; \Gamma\in \R^{n\times n}\label{eq:relaxG_const}.
		\end{align}
		\normalsize\end{subequations}
	\normalsize
We first prove that \eqref{eq:relaxG} is equivalent to \decomp (Proposition~\ref{prop:equivalentG}), and then show that \eqref{eq:relaxF} and \eqref{eq:relaxG} are equivalent (Proposition~\ref{prop:equivalent}). Before doing so, let us verify that \eqref{eq:relaxF}--\eqref{eq:relaxG} are indeed relaxations of \eqref{eq:optMmatrix}.
	
	\begin{proposition}\label{prop:valid}
		Problems \eqref{eq:relaxF}--\eqref{eq:relaxG} are valid convex relaxations of \eqref{eq:optMmatrix}.
	\end{proposition}
\begin{proof}
	We only prove this result for \eqref{eq:relaxG}; the proof for \eqref{eq:relaxF} follows from identical arguments and is omitted for brevity.
	
	First we argue convexity of \eqref{eq:relaxG}. Clearly, the objective \eqref{eq:relaxG_obj} is linear and constraints \eqref{eq:relaxG_const} are convex. Moreover, the right hand sides of constraints \eqref{eq:relaxG_pair} are supremum of convex functions, thus convex. 
	
	Now we argue that \eqref{eq:relaxG} is indeed a relaxation of \eqref{eq:optMmatrix}. Suppose that constraints $\Gamma_{ij}=x_ix_j$ and $z\in C$ are added to \eqref{eq:relaxG}: then $\langle \Gamma, Q\rangle=x'Qx$ and the objective functions of \eqref{eq:optMmatrix} and \eqref{eq:relaxG} coincide. Moreover, for any nonnegative $d_{ij}^i, d_{ij}^j$ such that $d_{ij}^i d_{ij}^j\geq 1$, we see that 
	\begin{align*}
	g(z_i,z_j,x_i,x_j; d_{ij}^i,d_{ij}^j)&\leq d_{ij}^i x_i^2-2x_{i}x_j+d_{ij}^jx_j^2 \tag{Theorem~\ref{theo:MHullGeneral} -- validity}\\
	&=d_{ij}^i \Gamma_{ii}-2\Gamma_{ij}+d_{ij}^j\Gamma_{jj}\tag{$\Gamma_{ij}=x_ix_j$},
	\end{align*}
thus inequalities \eqref{eq:relaxG_pair} are satisfied. So, if constraints $\Gamma_{ij}=x_ix_j$ and $z\in C$ are added, 
\eqref{eq:relaxG}  is equivalent to \eqref{eq:optMmatrix}. Hence,
\eqref{eq:relaxG} is a relaxation of \eqref{eq:optMmatrix}.
\end{proof}

	\begin{proposition} \label{prop:equivalentG}If $Q$ is a positive definite M-matrix, then problems 
		\decomp  and \eqref{eq:relaxG} are equivalent.
	\end{proposition}
\begin{proof}
Consider the variable $\Gamma_{ij}$ in \eqref{eq:relaxG} for some pair $i<j$: observe that it only appears in the objective with coefficient $Q_{ij}\leq 0$, and a single constraint \eqref{eq:relaxG_pair}. It follows that in an optimal solution of \eqref{eq:relaxG}, variable $\Gamma_{ij}$ is as large as possible and the corresponding constraint \eqref{eq:relaxG_pair} is binding:
	\begin{equation*}
	2\Gamma_{ij}=\max_{\substack{d_{ij}^id_{ij}^j\geq 1\\d_{ij}^i,d_{ij}^j\geq 0}}g(z_i,z_j,x_i,x_j; d_{ij}^i,d_{ij}^j)-\left(d_{ij}^i \Gamma_{ii}+d_{ij}^j\Gamma_{jj}\right).
	\end{equation*}
	Therefore, we find that problem \eqref{eq:relaxG} is equivalent to 
	\small \begin{subequations}\label{eq:relaxG2}
		\begin{align}
		\min_{z\in \bar C,x\in \R_+^n,\Gamma\in \R^n}\max_{d}\;&\|y\|_2^2-2y'x+\sum_{i=1}^n Q_{ii}\Gamma_{ii}\notag\\
		+\sum_{i=1}^n&\sum_{j=i+1}^n |Q_{ij}|\Big(g(z_i,z_j,x_i,x_j; d_{ij}^i,d_{ij}^j)-d_{ij}^i \Gamma_{ii}-d_{ij}^j\Gamma_{jj}\Big)\label{eq:relaxG2_obj}\\
		\text{s.t.}\;
		&d_{ij}^id_{ij}^j\geq 1, d_{ij}^i\geq 0,\; d_{ij}^j\geq 0 \quad \forall i<j.
		\end{align}
		\end{subequations}\normalsize
	
Rearranging terms, we see that the objective of the inner maximization problem \eqref{eq:relaxG2_obj} is equal to
\small\begin{equation*}
\sum_{i=1}^n \left(Q_{ii}-\sum_{j<i}|Q_{ji}|d_{ji}^i-\sum_{j>i}|Q_{ij}|d_{ij}^i\right)\Gamma_{ii}+\sum_{i=1}^n\sum_{j=i+1}^n |Q_{ij}|g(z_i,z_j,x_i,x_j;d_{ij}^i,d_{ij}^j),
\end{equation*} \normalsize
where we ignored the constant (in $d$) term  $\|y\|_2^2-2y'x$. In particular, the inner maximization problem is precisely the Lagrangian relaxation of \eqref{eq:Mseparation}, where $\Gamma_{ii}$ are the dual variables associated with constraints \eqref{eq:MsepCorrect}. Therefore, if strong duality holds for problem \eqref{eq:Mseparation}, then problems \decomp and \eqref{eq:relaxG} are equivalent. 

Finally, we verify that Slater's condition and, thus, strong duality for \eqref{eq:Mseparation} hold for positive definite $Q$. Since $Q$ is positive definite, we have that $Q=\bar Q+\rho I$ for an M-matrix $\bar Q$ (with same off-diagonals) and some $\rho>0$ (e.g., let $\rho$ be the minimum eigenvalue of $Q$). Since $\bar Q$ is an $M$-matrix, there exists a vector $\delta$ satisfying 
\begin{align*}
& \sum_{j<i}|Q_{ji}| \delta_{ji}^i+\sum_{j>i}|Q_{ij}| \delta_{ij}^i =Q_{ii}-\rho<Q_{ii}\quad &\forall i=1,\ldots,n \\
& \delta_{ij}^i\delta_{ij}^j\geq 1,\; \delta_{ij}^i\geq 0,\; \delta_{ij}^j\geq 0\quad &\forall i<j.
\end{align*}
It follows that letting $d_{ij}^i=\delta_{ij}^i+\epsilon$ and $d_{ij}^j=\delta_{ij}^j+\epsilon$ for all $i<j$ and $\epsilon>0$ small enough, we find a vector $d$ such that $d_{ij}^id_{ij}^j>1$ and \begin{equation}\label{eq:correct}\sum_{j<i}|Q_{ji}| d_{ji}^i+\sum_{j>i}|Q_{ij}| d_{ij}^i \leq Q_{ii} \quad \forall i=1,\ldots,n.\end{equation}
After increasing additional entries of $d$ until all inequalities \eqref{eq:correct} are tight, we find an interior point of \eqref{eq:Mseparation}.
\end{proof}
\normalsize

For the signal estimation problem,  $Q$ is positive-definite. Nonetheless, if strong duality does not hold,  formulation \eqref{eq:relaxG} is still a convex relaxation of \eqref{eq:optMmatrix} that is at least as strong as \texttt{decomp.}

\begin{proposition}\label{prop:equivalent} Problems \eqref{eq:relaxF} and \eqref{eq:relaxG} are equivalent.
\end{proposition}
\begin{proof}
For any $i<j$, we see from Theorem~\ref{theo:MHullGeneral} that constraint \eqref{eq:relaxG_pair} is equivalent to the pair of constraints:
\small\begin{align}0&\geq \max_{\substack{d_{ij}^id_{ij}^j\geq 1\\d_{ij}^i,d_{ij}^j\geq 0}}d_{ij}^if(z_i,z_j,x_i,\frac{x_j}{d_{ij}^i})+\frac{x_j^2}{z_j}\left(d_{ij}^j-\frac{1}{d_{ij}^i}\right)-d_{ij}^i \Gamma_{ii}+2\Gamma_{ij}-d_{ij}^j\Gamma_{jj}\label{eq:constJ}\\
0&\geq \max_{\substack{d_{ij}^id_{ij}^j\geq 1\\d_{ij}^i,d_{ij}^j\geq 0}}d_{ij}^jf(z_i,z_j,\frac{x_i}{d_{ij}^j},x_j)+\frac{x_i^2}{z_i}\left(d_{ij}^i-\frac{1}{d_{ij}^j}\right)-d_{ij}^i \Gamma_{ii}+2\Gamma_{ij}-d_{ij}^j\Gamma_{jj}.\label{eq:constI}
\end{align}\normalsize
Observe that $d_{ij}^if(z_i,z_j,x_i,\frac{x_j}{d_{ij}^i})\geq 0$ for any $d_{ij}^i>0$. Therefore, if $\frac{x_j^2}{z_j}>\Gamma_{jj}$, constraint \eqref{eq:constJ} is not satisfied since the right hand side can be made arbitrarily large by letting $d_{ij}^j\to \infty$ and $d_{ij}^i=1/d_{ij}^j$. Therefore, constraint \eqref{eq:constJ} implies that $\Gamma_{jj}z_j\geq x_j^2$. Similarly,  \eqref{eq:constI} implies that $\Gamma_{ii}z_i\geq x_i^2$.

Now assume that $\Gamma_{jj}z_j\geq x_j^2$ hold for all $j=1, \ldots, n$. In this case, for any optimal solution of the maximization problem \eqref{eq:constJ} we find that $d_{ij}^j$ is as small as possible; that is, $d_{ij}^j=1/d_{ij}^i$. Thus, if $\Gamma_{jj}z_j\geq x_j^2$ holds, then constraint \eqref{eq:constJ} reduces to
\begin{align}\label{eq:constJ2}0&\geq \max_{d_{ij}^i>0}d_{ij}^if(z_i,z_j,x_i,\frac{x_j}{d_{ij}^i})-d_{ij}^i \Gamma_{ii}+2\Gamma_{ij}-\frac{1}{d_{ij}^i}\Gamma_{jj},
\end{align}
which is precisely constraint \eqref{eq:relaxF_cuts}. Moreover, if $\Gamma_{ii}z_i\geq x_i^2$ holds, then constraint \eqref{eq:constI} reduces to
\begin{align}\label{eq:constI2}0&\geq \max_{d_{ij}^j>0}d_{ij}^jf(z_i,z_j,\frac{x_i}{d_{ij}^j},x_j)-\frac{1}{d_{ij}^j} \Gamma_{ii}+2\Gamma_{ij}-d_{ij}^j\Gamma_{jj}. 
\end{align}
After a change of variable $d_{ij}^i=1/d_{ij}^j$ and noting that $(1/d_{ij}^i)f(z_i,z_j,d_{ij}^ix_i,x_j)=d_{ij}^if(z_i,z_j,x_i,x_j/d_{ij}^i)$, we conclude that \eqref{eq:constI2} is equivalent to \eqref{eq:constJ2}.
\end{proof}

\begin{remark}
	Note that constraints \eqref{eq:relaxF_cuts}--\eqref{eq:relaxG_pair} are necessary only if $Q_{ij}\neq 0$. For the signal estimation problem \eqref{eq:denoising}, $Q_{ij}= 0$ for $\rev{\{i,j\}}\not\in A$. Thus, the methods developed here are particularly efficient when $Q$ is sparse. 
\end{remark}

	\subsection{Implementation via conic quadratic optimization}\label{sec:conic_quadratic} The objectives \eqref{eq:relaxF_obj} and \eqref{eq:relaxG_obj} are linear, and constraints \eqref{eq:relaxF_persp} are rotated cone constraints, and thus can be handled directly by conic quadratic optimization solvers. In Section~\ref{sec:socp_f}, we show how constraints \eqref{eq:relaxF_cuts} and \eqref{eq:relaxG_pair} can be reformulated as a conic constraints \emph{for a fixed value of $d_{ij}$}. Then we describe, in Section~\ref{sec:cutting}, a cutting plane method for implementing \eqref{eq:relaxF}.
	
	\subsubsection{Conic quadratic reformulation of functions $f$ and $g$}\label{sec:socp_f}
		
	We now show how to formulate convex models involving functions $f$ and $g$ as conic quadratic optimization problems. Specifically, we show how to model the epigraph of functions $f$ and $g$ in Propositions \ref{prop:cqExtendedXP} and \ref{prop:cqExtendedZP}, respectively.}
	\begin{proposition}[Extended formulation of $\conv(X^2)$]
		\label{prop:cqExtendedXP}
		A point $(z,x,s) \\ \in \conv(X^2)$ if and only if $(z,x,s)\in [0,1]^2\times \R_+^3$ and there exists $v,w\in \R$ such that the set of inequalities
		\begin{equation}
		\label{eq:systemXplus}
		v\geq x_1-x_2,\;v^2\leq sz_1,\;w\geq x_2-x_1,\; w^2\leq sz_2
		\end{equation}
		are satisfied.
	\end{proposition}
	\begin{proof}
		Suppose, without loss of generality, that $x_1\geq x_2$ and that $(z,x)$ satisfies the bound constraints. If $(z,x,s)\in \conv(X^2)$ then $\frac{(x_1-x_2)^2}{z_1}\leq s$; setting $v=x_1-x_2$ and $w=0$, we find a feasible solution for \eqref{eq:systemXplus}. Conversely, if \eqref{eq:systemXplus} is feasible, then $\frac{(x_1-x_2)^2}{z_1}\leq \frac{v^2}{z_1}\leq s$ and $(z,x,s)\in \conv(X^2)$. 
	\end{proof}
	
	\ignore{
		From Proposition~\ref{prop:cqExtendedXP}, 
		we see that the \texttt{pairwise} relaxation \eqref{eq:M-natOriginal} can be reformulated as
		\begin{align*}
		\sum_{i=1}^ny_i^2+\min\;&\sum_{i=1}^n\left(-2y_ix_i+s_i\right)+\lambda\sum_{\rev{\{i,j\}}\in A}s_{ij}\\
		\text{ s.t.}\;&\sum_{i=1}^nz_i\leq k\\
		(\texttt{M-nat}) \ \ \ \ \ \ \ \ \ \ &x_i^2\leq s_iz_i&i=1\ldots,n\\  
		&x_i\leq \|y\|_\infty z_i&i=1\ldots,n\\
		&v_{ij}\geq x_i-x_j,\; v_{ij}^2\leq s_{ij}z_i&\forall\rev{\{i,j\}}\in A \\
		&w_{ij}\geq x_j-x_i,\; w_{ij}^2\leq s_{ij}z_j&\forall\rev{\{i,j\}}\in A\\
		&z\in [0,1]^n,\; x\in \R_+^n,\; s\in \R_+^{n+|A|},\; v,w\in \R_+^{|A|}.
		\end{align*}
	}
	
	\begin{proposition}[Extended formulation of $\conv(Z^2)$]
		\label{prop:cqExtendedZP}
		A point $(z,x,s) \\ \in \conv(Z^2)$ if and only if $(z,x,s) \in [0,1]^2\times \R_+^3$ and there exists $s_1,s_2,q_1,q_2\in \R_+$ and $v_1,v_2,w_1,w_2\in \R_+$ such that the set of inequalities
		\begin{align*}
		x_1^2\leq s_1z_1,\; x_2^2\leq s_2z_2\tag{\persp}\\
		d_1v_1\geq d_1x_1-x_2,\; v_1^2\leq q_1z_1\tag{$z_1\geq z_2$ and $d_1x_1\geq x_2$}\\
		d_1v_2\geq -d_1x_1+x_2,\; v_2^2\leq q_1z_2\tag{$z_1\geq z_2$ and $d_1x_1\leq x_2$}\\
		d_1q_1+s_2\left(d_2-\frac{1}{d_1}\right)\leq s\tag{$z_1\geq z_2$}\\
		d_2w_1\geq x_1-d_2x_2,\; w_1^2\leq q_2z_1\tag{$z_1\leq z_2$ and $x_1\geq d_2x_2$}\\
		d_2w_2\geq -x_1+d_2x_2,\; w_2^2\leq q_2z_2\tag{$z_1\leq z_2$ and $x_1\leq d_2x_2$}\\
		d_2q_2+s_1\left(d_1-\frac{1}{d_2}\right)\leq s \tag{$z_1\leq z_2$}
		\end{align*}
		are satisfied.
	\end{proposition}
	\begin{proof}
		Follows from using the system \eqref{eq:systemXplus} with inequalities \eqref{eq:negativeDecomp2}--\eqref{eq:negativeDecomp1}.
	\end{proof}

\rev{	
	\subsubsection{Improved cutting surface method}\label{sec:cutting}
	
	\normalsize
	
Our implementation of the strong relaxation \decomp is based on formulation \eqref{eq:relaxF}, implemented in a cutting surface method. Consider the relaxation of \eqref{eq:relaxF} given by 
\begin{subequations}\label{eq:relaxFfinite}
	\small 
	\begin{align}
	\min_{(z,x,\Gamma)}\;&\|y\|_2^2-2y'x+\langle \Gamma, Q\rangle\\
	\text{s.t.}\;
	& \Gamma_{ii}z_i\geq x_i^2&\hspace{-6cm}\forall i=1,\ldots,n\hfill \\
	& 0\geq d\cdot f(z_i,z_j,x_i,\frac{x_j}{d})-\left(d \Gamma_{ii}-2\Gamma_{ij}+\frac{1}{d}\Gamma_{jj}\right)&\forall i<j, \; \forall d\in \Delta_{ij} \hfill\label{eq:relaxFfinite_cuts}\\
	& 0\leq x_i\leq \|y\|_\infty z_i&\hspace{-3cm}i=1,\ldots,n\hfill\\
	&z\in \bar C,\; x\in \R_+^n,\; \Gamma\in \R^{n\times n},
	\end{align}
	\normalsize\end{subequations}
where each $\Delta_{ij}$ is a finite subset of $\R$. From Proposition~\ref{prop:cqExtendedXP}, each constraint \eqref{eq:relaxFfinite_cuts} can be formulated by introducing new variables $s,v,w\geq 0$ as the system
\begin{subequations}\label{eq:cut}
\begin{align} 
&0\geq d s-\left(d \Gamma_{ii}-2\Gamma_{ij}+\frac{1}{d}\Gamma_{jj}\right),\\
&v\geq x_1-\frac{x_2}{d},\;v^2\leq sz_1,\; w\geq \frac{x_2}{d}-x_1,\; w^2\leq sz_2.
\end{align} 
\end{subequations}
Therefore, relaxation \eqref{eq:relaxFfinite} can be solved using a conic quadratic solver. 

In the proposed cutting surface method, formulation \eqref{eq:relaxFfinite} is iteratively refined by adding additional elements to sets $\Delta_{ij}$, as outlined in Algorithm~\ref{alg:Mmatrix}. First, all sets $\Delta_{ij}$ are initialized to the singleton $\{1\}$ (line~\ref{line:init}). At each iteration of the algorithm, a relaxation of the form \eqref{eq:relaxFfinite} is solved to optimality (line~\ref{line:solve}). Then, for each pair of indexes $i<j$ where the relaxation induced by \eqref{eq:relaxFfinite_cuts} is weak, the set $\Delta_{ij}$ is enlarged to improve the relaxation (line~\ref{line:refine});  Remark~\ref{rem:explicit} and Proposition~\ref{prop:optimalD} below show to efficiently check whether the relaxation needs to be refined and how to do so, respectively.

\begin{algorithm}[h] \small
	\caption{Algorithm to solve formulation \decomp}
	\label{alg:Mmatrix}
	\begin{algorithmic}[1]
		\renewcommand{\algorithmicrequire}{\textbf{Input:}}
		\renewcommand{\algorithmicensure}{\textbf{Output:}}
		\Ensure $(\hat x,\hat z,\hat\Gamma)$ optimal for \decomp
		
		\State $\Delta_{ij}\leftarrow\{1\}$ for all $i<j$ \label{line:init} 
		\While{Stopping criterion not met}\label{line:stopping}
		\State $(\hat x,\hat z,\hat\Gamma)\leftarrow $ Solve \eqref{eq:relaxFfinite}\label{line:solve}
		\For{all $i<j$}
		\If{Constraint \eqref{eq:relaxF_cuts} is not satisfied} 
		\State Compute optimal $d_{ij}^*$ for maximization \eqref{eq:relaxF_cuts}\Comment{See Proposition~\ref{prop:optimalD}}
		\State $\Delta_{ij}\leftarrow \Delta_{ij}\cup \{d_{ij}^*\}$\label{line:refine}
		\EndIf
		
		\EndFor	
		\EndWhile \label{line:endWhile1}
		\State \Return $(\hat x,\hat z,\hat\Gamma)$
	\end{algorithmic}
\end{algorithm}

\begin{proposition}\label{prop:optimalD}
For any $i<j$, the optimal solution of the inner maximization problem \eqref{eq:relaxF_cuts} is obtained as follows:
\begin{enumerate}
	\item If $x_i^2/\Gamma_{ii}\geq x_j^2/\Gamma_{jj}$ then:
	\begin{enumerate}
		\item If $\Gamma_{ii}-x_i^2/z_i=0$, then $d_{ij}\to\infty$ is optimal.
		\item Otherwise, $d_{ij}=\sqrt{\frac{\Gamma_{jj}-\frac{x_j^2}{z_i}}{\Gamma_{ii}-\frac{x_i^2}{z_i}}}$ is optimal.
	\end{enumerate}
\item If $x_i^2/\Gamma_{ii}\leq x_j^2/\Gamma_{jj}$ then:
\begin{enumerate}
	\item If $\Gamma_{ii}-x_i^2/z_j=0$, then $d_{ij}\to\infty$ is optimal.
	\item Otherwise, $d_{ij}=\sqrt{\frac{\Gamma_{jj}-\frac{x_j^2}{z_j}}{\Gamma_{ii}-\frac{x_i^2}{z_j}}}$ is optimal.
\end{enumerate}
\end{enumerate}

\end{proposition}
\begin{proof}
Suppose $x_i^2/\Gamma_{ii}\geq x_j^2/\Gamma_{jj}$. Note that $\Gamma_{ii}\geq x_i^2/z_i$ holds from constraints \eqref{eq:relaxF_persp}. Moreover, we find that 
\begin{align*}
\Gamma_{jj}-\frac{x_j^2}{z_i}\geq \Gamma_{ii}\frac{x_j^2}{x_i^2}-\frac{x_j^2}{z_i}=x_j^2\left(\frac{\Gamma_{ii}}{x_i^2}-\frac{1}{z_i}\right)\geq 0.
\end{align*}
We now show that there exists a stationary point of \eqref{eq:relaxF_cuts} satisfying $d_{ij}x_i\geq x_j$. In this case, optimization problem \eqref{eq:relaxF_cuts} reduces to
\small\begin{align}
&0\geq \max_{d_{ij}>0}\frac{d_{ij}x_i^2-2x_ix_j+x_j^2/d_{ij}}{z_i}-\left(d_{ij}\Gamma_{ii}-2\Gamma_{ij}+\frac{1}{d_{ij}}\Gamma_{jj}\right)\notag\\
\Leftrightarrow\;&0\geq 2\left(\Gamma_{ij}-\frac{x_ix_j}{z_i}\right)+\max_{d_{ij}>0}\left\{-d_{ij}\left(\Gamma_{ii}-\frac{x_i^2}{z_i}\right)-\frac{1}{d_{ij}}\left(\Gamma_{jj}-\frac{x_j^2}{z_i}\right)\right\}.\label{eq:optimization_d}
\end{align}\normalsize
If $\Gamma_{ii}-x_i^2/z_i=0$, then $d_{ij}^*\to\infty$ is an optimal solution to \eqref{eq:optimization_d}. If $\Gamma_{jj}-x_j^2/z_i=0$, then $d_{ij}^*=0$ is optimal. Moreover, if both $\Gamma_{ii}-x_i^2/z_i>0$ and $\Gamma_{jj}-x_j^2/z_i>0$, then taking derivatives with respect to $d_{ij}$ we find that 
\begin{equation}\label{eq:optDij}d_{ij}^*=\sqrt{\frac{\Gamma_{jj}-\frac{x_j^2}{z_i}}{\Gamma_{ii}-\frac{x_i^2}{z_i}}} \cdot \end{equation}
Finally, we verify that the condition $d_{ij}^*x_i\geq x_j$ holds. Indeed, this condition reduces to 
\begin{align*}
\left(\frac{\Gamma_{jj}-\frac{x_j^2}{z_i}}{\Gamma_{ii}-\frac{x_i^2}{z_i}}\right)x_i^2\geq x_j^2\Leftrightarrow \left(\Gamma_{jj}-\frac{x_j^2}{z_i}\right)x_i^2\geq \left(\Gamma_{ii}-\frac{x_i^2}{z_i}\right)x_j^2\Leftrightarrow \frac{x_i^2}{\Gamma_{ii}}\geq \frac{x_j^2}{\Gamma_{jj}},
\end{align*}
which is satisfied. The proof for the case $x_i^2/\Gamma_{ii}\leq x_j^2/\Gamma_{jj}$ is analogous.
\end{proof}

\normalsize

\begin{remark}\label{rem:explicit}
	By replacing $d_{ij}$ with its optimal value in \eqref{eq:relaxF_cuts}, we find that this constraint can be written explicitly as the piecewise constraint
	\begin{equation}\label{eq:explicit}
	0\geq \begin{cases}
	 \Gamma_{ij}-\frac{x_ix_j}{z_i}-\sqrt{\left(\Gamma_{ii}-\frac{x_i^2}{z_i}\right)\left(\Gamma_{jj}-\frac{x_j^2}{z_i}\right)}	& \text{if }\frac{x_i^2}{\Gamma_{ii}}\geq \frac{x_j^2}{\Gamma_{jj}}\\
	 \Gamma_{ij}-\frac{x_ix_j}{z_j}-\sqrt{\left(\Gamma_{ii}-\frac{x_i^2}{z_j}\right)\left(\Gamma_{jj}-\frac{x_j^2}{z_j}\right)}	& \text{if }\frac{x_i^2}{\Gamma_{ii}}\leq \frac{x_j^2}{\Gamma_{jj}}.
	\end{cases}
	\end{equation}
However, constraint \eqref{eq:explicit} is not conic quadratic. 
\end{remark}

\begin{remark}
	In our computations, we use the following stopping criterion in line~\ref{line:solve}. Let $\zeta_{\text{old}}$ and $\zeta_{\text{new}}$ be the optimal objective value of the relaxation (line~\ref{line:solve}). 
	The	algorithm is terminated when the relative improvement of the relaxation  $\left(\zeta_{\text{new}}-\zeta_{\text{old}}\right)/\zeta_{\text{new}} \leq 5\times 10^{-5}$.
\end{remark}

\normalsize

	\begin{remark}
		Using Proposition~\ref{prop:cqExtendedZP}, one can extend the ideas discussed in this section to tackle \eqref{eq:relaxG} in a conic quadratic optimization framework as well. However, we prefer formulation \eqref{eq:relaxF} since the conic quadratic representation of function $f$ is simpler and more compact.
	\end{remark}

\subsection{Lagrangian methods for estimation with regularized objective}\label{sec:decomposition} 
The cutting surface method introduced in Section~\ref{sec:conic_quadratic} requires solving a sequence of progressively larger conic quadratic optimization problems. 
 Based on our computations, this method can handle a variety of constraints (encoded by set $\bar C$), and solve 
 the instances with $n\leq 10,000$ within seconds. 
 For better scalability,
  in this section, we develop a Lagrangian relaxation-based method for the estimation problem with regularization objective:
\begin{subequations}\label{eq:1D} \small
\begin{align}
\min_{x, z}\;&\|y-x\|_2^2+\lambda\sum_{i=1}^{n-1}(x_{i+1}-x_i)^2+\mu\sum_{i=1}^n z_i\\
\text{s.t.}\;& 0\leq x_i\leq \|y\|_\infty z_i&i=1,\ldots,n\\
& x\in \R_+^n,\; z\in \{0,1\}^n
\end{align}
\end{subequations} \normalsize
where $\mu\geq 0$ is a regularization parameter controlling the sparsity of the target signal. Let $L=\left\{\ell_1,\ldots,\ell_m,\ell_{m+1}\right\}\subseteq \{1,\ldots,n\}$ be any subset of the indexes such that $1=\ell_1<\ldots< \ell_m<\ell_{m+1}=n+1$. With the introduction of additional variables $w_j=x_{\ell_j}-x_{\ell_j-1}$, problem \eqref{eq:1D} can be equivalently written as
\small\begin{subequations}\label{eq:1Dw}
	\begin{align}
	\min_{x, z}\;&\sum_{j=1}^m\left(\sum_{i=\ell_j}^{\ell_{j+1}-1}(y_i-x_i)^2+\lambda\sum_{i=\ell_j}^{\ell_{j+1}-2}(x_{i+1}-x_i)^2
	+\mu\sum_{i=\ell_j}^{\ell_{j+1}-1}z_i\right)+\sum_{j=2}^m w_j^2\\
	\text{s.t.}\;& w_j=x_{\ell_j}-x_{\ell_j-1}&\hspace{-3cm} j=2,\ldots,m\quad\quad \label{eq:1Dw_lagrang}\\
	& 0\leq x_i\leq \|y\|_\infty z_i&\hspace{-3cm}i=1,\ldots,n\quad\quad \\
	& x\in \R_+^n,\; z\in \{0,1\}^n,\; w\in \R^{m-1}.
	\end{align}
\end{subequations}\normalsize

Without the coupling constraints \eqref{eq:1Dw_lagrang}, problem \eqref{eq:1Dw} decomposes into $m$ independent problems, each with variables indexed in $[\ell_j, \ell_{j+1}-1]$ for $j=1,\ldots,m$. Letting $\gamma_j$ be the Lagrange multiplier for constraint $w_j=x_{\ell_j}-x_{\ell_j-1}$, we obtain the Lagrangian dual problem
\begin{subequations}\label{eq:1DlagrangePre} \small
	\begin{align}
	\max_{\gamma\in \R^{m-1}}\min_{x, z,w}\;&\sum_{j=1}^m\bigg (\sum_{i=\ell_j}^{\ell_{j+1}-1}(y_i-x_i)^2+\lambda \! \! \sum_{i=\ell_j}^{\ell_{j+1}-2}(x_{i+1}-x_i)^2+
	\mu\! \! \sum_{i=\ell_j}^{\ell_{j+1}-1}z_i\bigg)+ \!\sum_{j=2}^m w_j^2\notag\\
	&+\gamma_j(w_j-x_{\ell_j}+x_{\ell_j-1})\\
	\text{s.t.}\;
	& 0\leq x_i\leq \|y\|_\infty z_i&\hspace{-10cm}i=1,\ldots,n\quad \quad\quad \\
	& x\in \R_+^n,\; z\in \{0,1\}^n,\; w\in \R^{m-1}.
	\end{align}
\end{subequations}
\normalsize
Observe that $w_j=-\gamma_j/2$ holds for an optimal solution of the inner minimization problem. Moreover, to obtain a strong convex relaxation, we can reformulate each independent inner minimization problem using the formulations discussed in Section~\ref{sec:Mseparation}, yielding the convex relaxation
\begin{subequations}\label{eq:1Dlagrange}
	\begin{align}
	\max_{\gamma\in \R^{m-1}}\min_{x, z}\;&\sum_{j=1}^m\left(\sum_{i=\ell_j}^{\ell_{j+1}-1}(y_i^2-2y_ix_i)+\theta_j(z,x)+\mu\sum_{i=\ell_j}^{\ell_{j+1}-1}z_i\right)-\sum_{j=2}^m\frac{\gamma_j^2}{4}\notag\\
	&+\gamma_j(x_{\ell_j-1}-x_{\ell_j})\label{eq:1Dlagrange_obj}\\
	\text{s.t.}\;
	& 0\leq x_i\leq \|y\|_\infty z_i&\hspace{-10cm}i=1,\ldots,n\quad \quad\quad \\
	& x\in \R_+^n,\; z\in [0,1]^n,
	\end{align}
\end{subequations}
where $\theta_j(z,x)$ is the convexification of the epigraph of the term $$\sum_{i=\ell_j}^{\ell_{j+1}-1}x_i^2+\lambda\sum_{i=\ell_j}^{\ell_{j+1}-2}(x_{i+1}-x_i)^2.$$

\paragraph{\textbf{Implementation}} Problem \eqref{eq:1Dlagrange} can be solved via a primal-dual method: for any fixed $\gamma$, the inner minimization problem can be solved by solving $m$ independent sub-problems (in parallel), and each sub-problem is solved using Algorithm~\ref{alg:Mmatrix}. We now describe our implementation of the ``main" outer maximization problem.
\begin{description}
	\item [Set $L$] Given a target number of subproblems $m\in \Z_+$, we let $\ell_j=1+(j-1)\lfloor n/m\rfloor$ for $j=1,\ldots,m$. 
	\item [Subgradient method] Given $\gamma\in \R^{m-1}$, a subgradient of the objective \eqref{eq:1Dlagrange_obj} at $\gamma$ is given by 
	$\xi(\gamma)_j=-\frac{\gamma_j}{2}+(x_{\ell_j-1}^*-x_{\ell_j}^*)$, where $x^*$ is an optimal solution of the inner minimization problem at $\gamma$. Thus, letting $\gamma^h$ be the value of $\gamma$ at iteration $h\in Z_+$, we use the update rule $\gamma^{h+1}=\gamma_h+(1/h)\xi(\gamma^h)$. 
	\item[Initial point] We start the algorithm with the initial point $\gamma^0=0$. Note that if $x_{\ell_j-1}=x_{\ell_j}=0$ (which is the case for large $\mu$), then $\gamma_j = 0$ is optimal. 
	\item[Stopping criterion] We terminate the algorithm when $\|\xi(\gamma^h)\|_\infty < \epsilon$ ($\epsilon=10^{-3}$ in our computations) or when the number of iterations reaches $h_{max}$ ($h_{max}=100$ in our computations). 
	\item[Additional considerations] In the first iteration, we need to solve $m$ subproblems. However, the subsequent iterations often require solving fewer subproblems: if $\xi(\gamma^h)_j= \xi(\gamma^{h+1})_j$ and $\xi(\gamma^h)_{j+1}= \xi(\gamma^{h+1})_{j+1}$, then at iteration $h+1$ solution of subproblem $j$ does not change from the previous iteration. For problem instances with large $\mu$, the number of subproblems solved in subsequent iterations reduces considerably. 
\end{description}
}

\new{\paragraph{\textbf{On Lagrangian methods for general adjacency graphs}}We point out that the tailored Lagrangian relaxation-based decomposition can be applied to higher-dimensional adjacency graphs. For example, when the adjacency graph is a two-dimensional mesh, one can divide the mesh into rectangular regions and apply a similar decomposition approach. Theoretically, it is possible to use a similar method for arbitrary adjacency graphs. 
However, key to efficiency of the Lagrangian method, described in ``additional considerations" above, is the observation that if the dual variables corresponding to the border of a subregion do not change, then subproblems need not be resolved. With higher-dimensional graphs, where borders are no longer defined by two points but by rectangles or other objects, more subproblems may need to be resolved \emph{unless regions are defined carefully}, depending on the data $y$.
Developing such methods to partition the adjacency graph is beyond the scope of this paper. 
}

\section{Computations}
\label{sec:computations}
In this section we present experiments with utilizing the strong convex relaxations based on the pairwise convexification methods proposed in the paper. \rev{In Section~\ref{sec:accelerometer}, we perform experiments to evaluate whether the convex model \decomp provides a good approximation to the non-convex problem \eqref{eq:optMmatrix}. In Section~\ref{sec:syntData} we test the merits of formulation \decomp  (with a variety of constraints $\bar C$) compared to the usual $\ell_1$-approximation from an inference perspective. Finally, in Section~\ref{sec:comp_decomposition} we test the Lagrangian relaxation-based method proposed in Section~\ref{sec:decomposition}. We use Mosek 8.1.0 (with default settings) to solve the conic quadratic optimization problems. All computations are performed on a laptop with eight Intel(R) Core(TM) i7-8550 CPUs and 16GB RAM. All data \new{and code} used in the computations are available at \texttt{https://sites.google.com/usc.edu/gomez/data}.}

\subsection{Relaxation quality}\label{sec:accelerometer}
\rev{This section is devoted to testing how well the proposed convex relaxations are able to approximate the $\ell_0$ optimization problems using real data. }

\subsubsection{Data}
Consider the accelerometer data depicted in Figure~\ref{fig:noisyAccelo} (A), used in \cite{casale2011human,casale2012personalization} and downloaded from the UCI Machine Learning Repository \cite{Dua:2017}. The time series corresponds to the ``x acceleration" of participant 2 of the ``Activity Recognition \rev{from} Single Chest-Mounted Accelerometer Dataset".  This participant was ``working at computer" until time stamp 44,149; ``standing up, walking and going upstairs" until time stamp 47,349; ``standing" from time stamp 47,350 to 58,544, from 80,720 to 90,439,  and from time 90,441 to 97,199; ``walking" from 58,545 to 80,719; ``going up or down stairs" from 90,440 to 94,349; ``walking and talking with someone" from 97,200 to 104,300; and ``talking while standing" from 104,569 to 138,000 (status between 104,301 and 104,568 is unknown). 
\begin{figure}[h!]
	\begin{adjustbox}{minipage=1\linewidth,scale=1.45}
		
		\subfloat[Original data]{\includegraphics[width=0.33\textwidth,trim={9.5cm 6cm 9.5cm 6cm},clip]{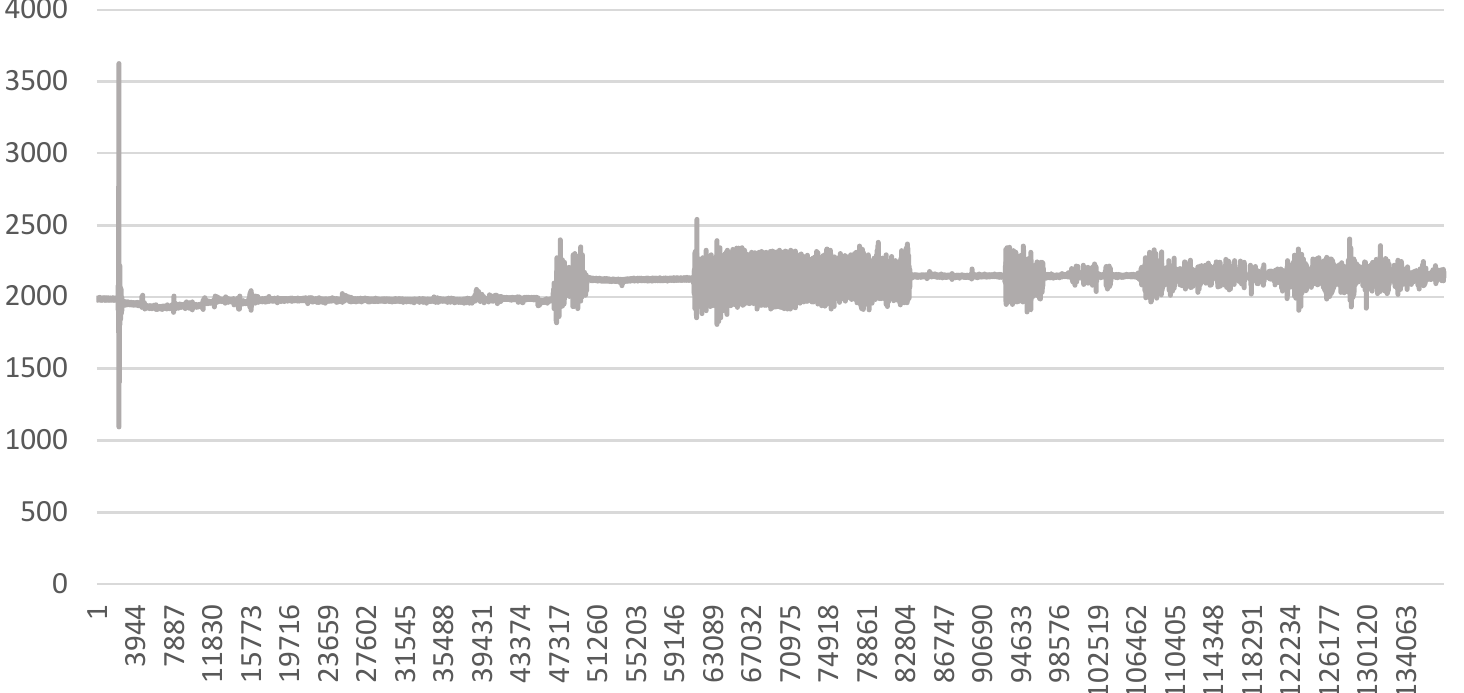}}\quad
		\subfloat[Transformed data]{\includegraphics[width=0.33\textwidth,trim={9.5cm 6cm 9.5cm 6cm},clip]{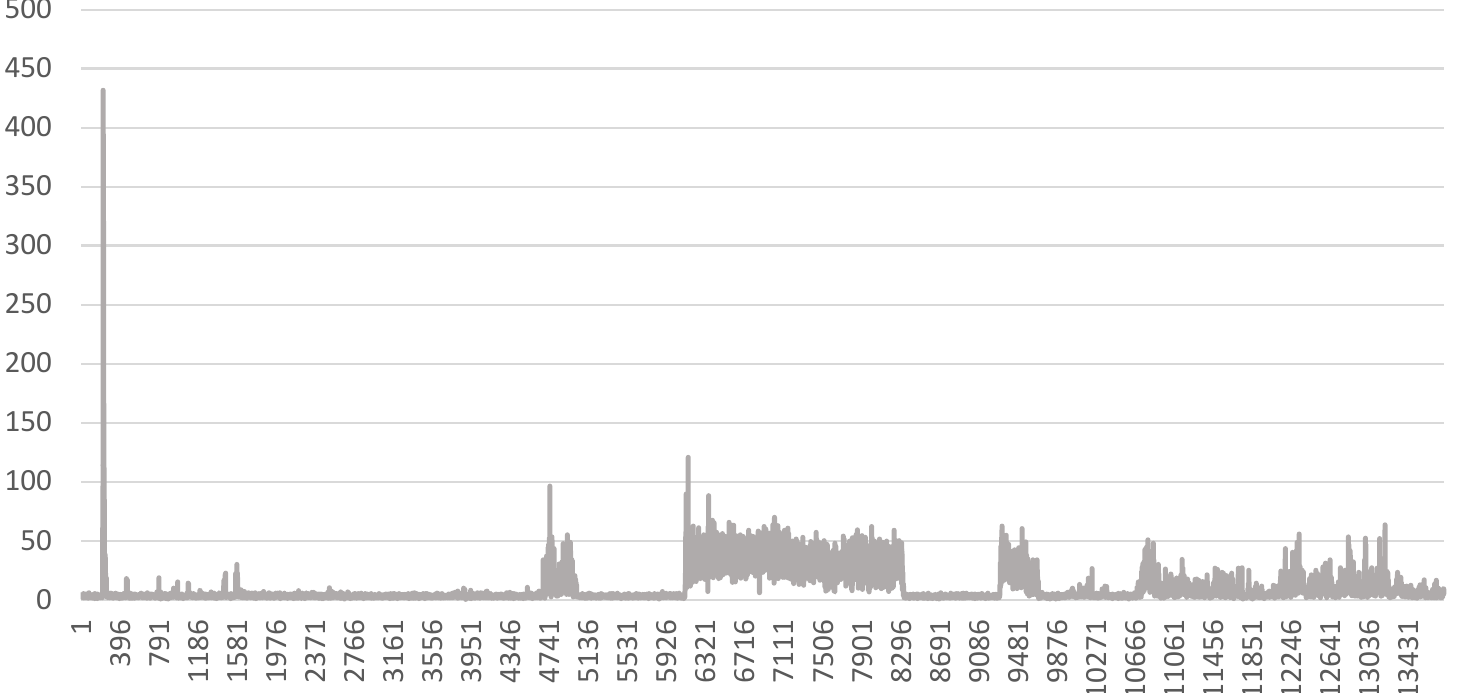}}
	\end{adjustbox}
	\caption{Underlying signals and noisy observations.}
	\label{fig:noisyAccelo}
\end{figure}

Several machine learning methods have been proposed to use accelerometer data to discriminate between activities, e.g., see \cite{bao2004activity} and the references therein. Variations of the acceleration can help to discriminate between activities \cite{casale2011human}. Moreover, as pointed out in \cite{wilson2008prying}, behaviors can be identified (at a simplistic level) from frequencies and amplitudes of wave patterns in a single axis of the accelerometer. Therefore, we consider a rudimentary approach to identify activities from the accelerometer data: we partition the dataset into windows of 10 samples each, and for each window we compute the mean absolute value of the successive differences, obtaining the dataset plotted in Figure~\ref{fig:noisyAccelo} (B)\footnote{One of the key features identified in \cite{casale2011human} for activity recognition are the minmax sums of 52-sample windows, computed as the sums of successive differences of consecutive ``peaks". The time series we obtain follows a similar intuition, but is larger and noisier due to smaller windows.}. \rev{Finally, we scale the data so that $\|y\|_\infty=1$.}

Given an optimal solution $x^*$ of the estimation problem \eqref{eq:denoisingMIO} or a suitable relaxation of it, periods with little or no physical activity can be naturally associated with time stamps $i$ where $x_i^*=0$, and values $x_i^*>0$ can be used as a proxy for the energy expenditure due to physical activity \cite{shepard2008identification}.  

\rev{
\subsubsection{Methods} We compare the following two relaxations of the $\ell_0$-problem 
\begin{subequations}\label{eq:l0Card}
\begin{align}
\min_{x\in \R_+^n}\;&\sum_{i=1}^n(y_i-x_i)^2+\lambda\sum_{i=1}^{n-1}(x_{i+1}-x_i)^2\\
\text{s.t.}\;&\sum_{i=1}^n z_i\leq k\\
& x \leq z, \\ 
& z\in \{0,1\}^n.
\end{align}
\end{subequations}
\begin{description}
	\item[L1] The natural convex relaxation of \eqref{eq:l0Card}, obtained by relaxing the integrality constraints to $z\in [0,1]^n$.
	\item[Decomp] The convex model \decomp  --equivalently, \eqref{eq:relaxF}-- implemented using Algorithm~\ref{alg:Mmatrix}.
\end{description}
}

The convex formulations used are relaxations of the $\ell_0$ problem; so, their optimal objective values $\zeta_{\text{LB}}$ provide lower bounds on the optimal objective value $\zeta$ of \eqref{eq:l0Card}. We use a simple \emph{thresholding} heuristic to construct a feasible solution for \eqref{eq:l0Card}: for a given solution $\hat x$ to a convex relaxation, let $\hat x_{(k)}$ denote the $k$-th largest value, and $\bar{x}$ be the solution given by 
$$\bar{x}_i=\begin{cases}\hat x_i&\text{if }\hat x_i\geq \hat x_{(k)}\\ 0 &\text{otherwise.}\end{cases}$$
By construction $\bar x$ is feasible for \eqref{eq:l0Card}, and its objective value $\zeta_{\text{UB}}$ provides an upper bound on $\zeta$. Thus, the optimality gap of the heuristic is 
\begin{equation}
\label{eq:gap}
\texttt{gap}=100 \times \frac{\zeta_{\text{UB}}-\zeta_{\text{LB}}}{\zeta_{\text{UB}}}\cdot.
\end{equation}

\rev{
	\subsubsection{Results}
We test the convex formulations with the accelerometer data using $\lambda = 0.1t$ and $k=500t$ for $t=1, \ldots,10$
 for all 100 combinations. Figure~\ref{fig:gaps1}(A) presents the optimality gaps obtained by each method for each value of $\lambda$ (averaging over all values of $k$), and Figure~\ref{fig:gaps1}(B) presents the optimality gaps for each value of $k$ (averaging over all values of $\lambda$). We see that \decomp substantially improves upon the natural $\ell_1$ relaxation. Indeed, the gaps from the $\ell_1$ relaxation are 66.7\% on average, and can be very close to 100\%; in contrast, the strong relaxations derived in this paper yields optimality gaps of 0.4\% on average.

\begin{figure}[h!]
	\subfloat[Gap as a function of $\lambda$.]{\includegraphics[width=0.49\textwidth,trim={11cm 6cm 11cm 5.5cm},clip]{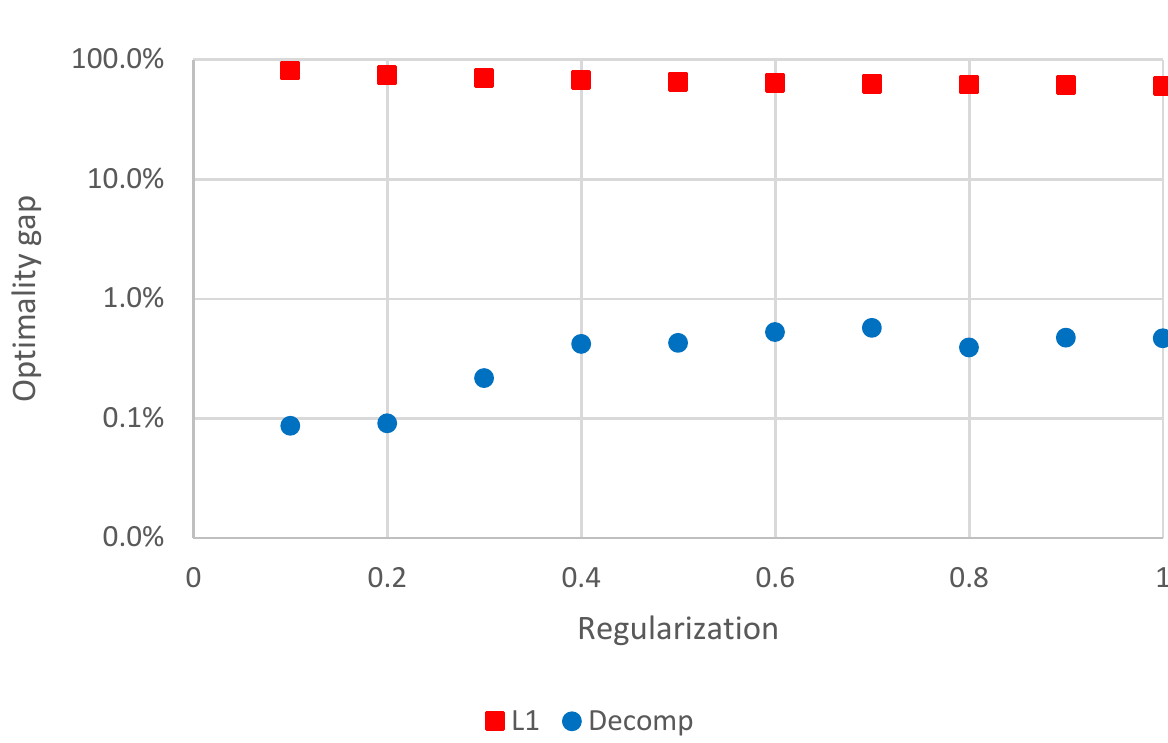}}\hfill		\subfloat[Gap as a function of $k$.]{\includegraphics[width=0.49\textwidth,trim={11cm 6cm 11cm 5.5cm},clip]{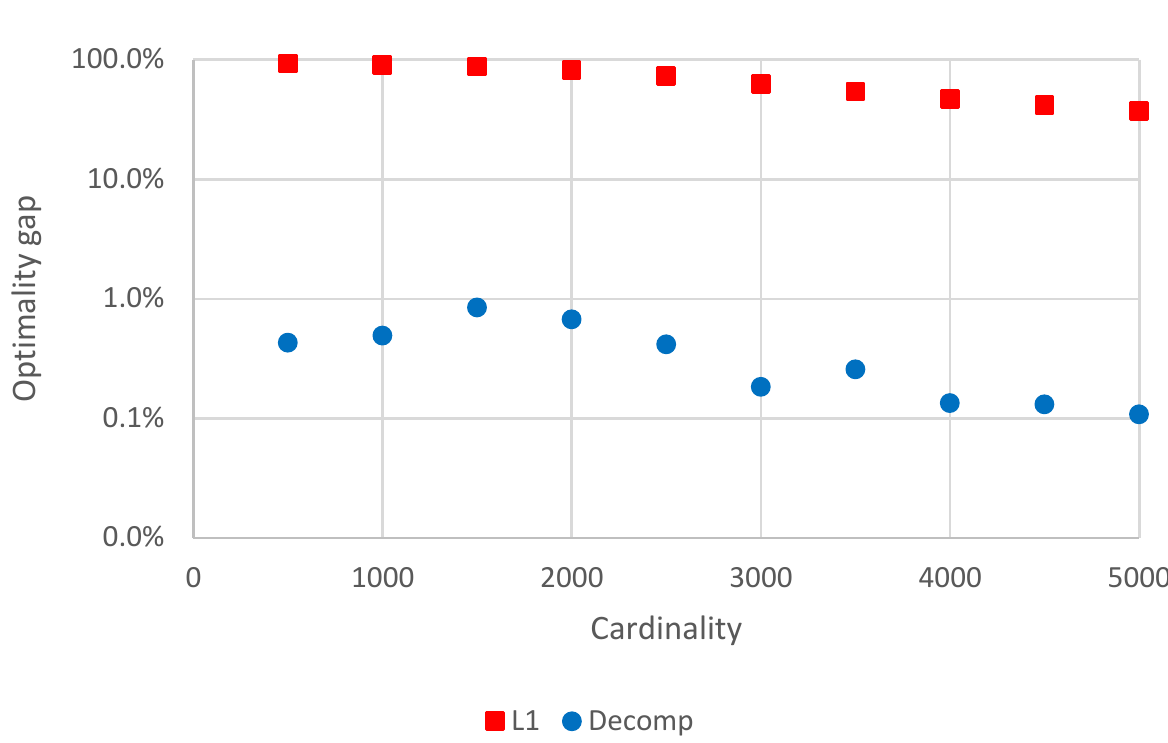}}
	\caption{Optimality gaps of the $\ell_1$ relaxation (red) and the proposed convexification \decomp (blue). }
	\label{fig:gaps1}
\end{figure}

Figure~\ref{fig:time_accelerometer} presents the distribution of the time required for each method to solve the respective convex model. We see that the improvement of relaxation quality of the new relaxations comes at the cost of computational efficiency: while the $\ell_1$ relaxation is solved in approximately one second, the proposed convexification requires on average 54 seconds. Although the vast majority of the instances are solved under 100 seconds using Algorithm~\ref{alg:Mmatrix}, a couple of instances require close to 10 minutes. Nonetheless, an average time of under minute to solve the instances to near-optimality (less than 1\% optimality gap) is adequate for most practical settings. 
Moreover, as shown in Section~\ref{sec:comp_decomposition}, the
computation times can be improved substantially using the decomposition method proposed in Section~\ref{sec:decomposition}.

\begin{figure}[h!]
\includegraphics[width=0.60\textwidth,trim={10cm 5cm 10cm 5cm},clip]{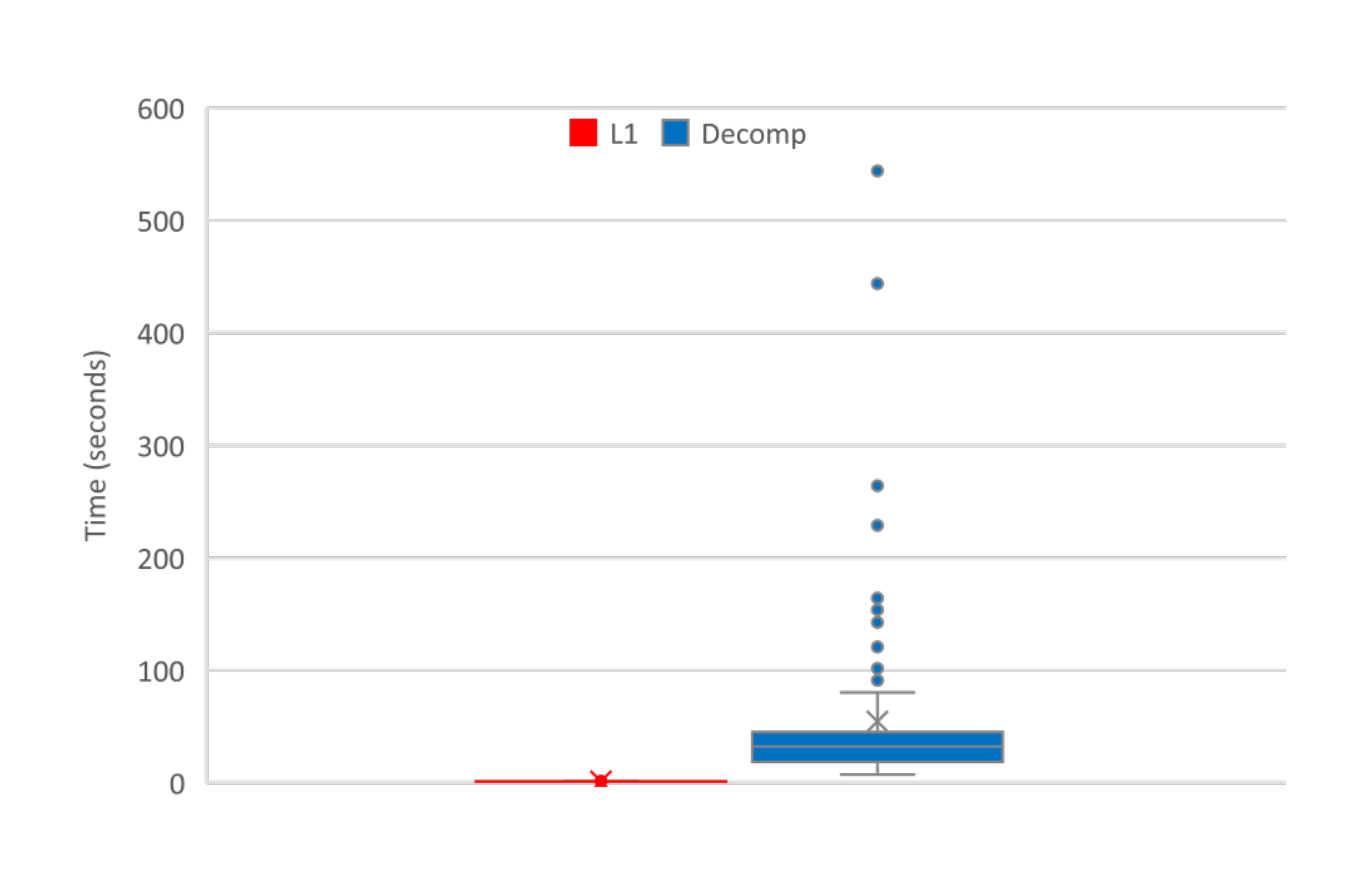}
	\caption{Distribution of CPU times for each method.}
	\label{fig:time_accelerometer}
\end{figure}
}

\new{Finally, we present a brief comparison with mixed-integer optimization methods. Specifically, we use the perspective reformulation of \eqref{eq:l0Card}, i.e., 
\begin{align*}
\min_{x\in \R_+^n}\;&\sum_{i=1}^n\left(y_i^2-2y_ix_i+\frac{x_i^2}{z_i}\right)+\lambda\sum_{i=1}^{n-1}(x_{i+1}-x_i)^2\\
\text{s.t.}\;&\sum_{i=1}^n z_i\leq k\\
& x \leq z, \\ 
& z\in \{0,1\}^n
\end{align*}
and solve the problems using Gurobi 8.0 with a one hour time limit. Table~\ref{tab:mio} presents the results: it shows, for $k\in \{2000,4000\}$ and $\lambda\in \{0.1,0.2\}$, the gaps and solution times for \texttt{$\ell_1$-approx}, \decomp, and the branch-and-bound solver (\texttt{b\&b})as well as the number of branch-and-bound nodes explored. The mixed-integer optimizer is unable to solve to optimality instances of this size within the time limit of one hour. While Gurobi is able to produce good optimality gaps, improving substantially upon the simple $\ell_1$-relaxation (in part due to the use of the perspective reformulation), \decomp produces better optimality gaps by an order-of-magnitude and requires only a small fraction of the computational effort.}

\begin{table}[h!]
	\caption{\new{Comparison with exact branch-and-bound method.}}
	\label{tab:mio}
	\begin{tabular}{c c |c c | c c | c c c}
		\hline
		\multirow{2}{*}{$k$}&\multirow{2}{*}{$\lambda$}&\multicolumn{2}{c|}{{\lasso}}&\multicolumn{2}{c|}{{\decomp}}&\multicolumn{3}{c}{{\texttt{b\&b}}}\\
		&&gap&time(s)&gap&time(s)&gap&time(s)&nodes\\
		\hline
		\multirow{2}{*}{2000}& 0.1 & 91.2 & 1 & 0.3 & 60 & 2.7 & 3,600 & 4,047\\
		& 0.2 & 87.0& 1 & 0.6 & 42 & 7.4 & 3,600 & 1,445\\
		&&&&&&&&\\
		\multirow{2}{*}{4000}& 0.1 & 68.0& 1 &0.0 & 7 & 3.4 & 3,600 & 1,350\\
		& 0.2 & 56.7& 1 &0.1 & 19 & 2.8 & 3,600 & 4,268\\
		\hline
	\end{tabular}
\end{table}

\rev{
\subsection{Statistical performance -- modeling with priors}\label{sec:computations_synt} In Section~\ref{sec:accelerometer} we established that the convex model derived in this paper indeed provides a much closer approximation for $\ell_0$ signal estimation problem 
than the usual $\ell_1$ relaxation. In this section we demonstrate that using the proposed convexification leads to better statistical performance than relying on the $\ell_1$-approximation alone. We also show how additional priors other than sparsity can be seamlessly integrated into the new convex models, and the benefits of doing so.
\label{sec:syntData}

\subsubsection{Data}
We now describe how we generate test instances.  First, the ``true" sparse signal $\hat y$ is generated as follows. Let $n$ be the number of time epochs, let $s$ be a parameter controlling the number of ``spikes" of the signal and let $h$ be a parameter controlling the length of each spike. Initially, the true signal is fully sparse, $\hat y=0$. Then we iteratively repeat the following process to generate $s$ spikes of non-zero values:
\begin{enumerate}
	\item We select an index $\ell$ uniformly between $1$ and $n+1-h$, corresponding to the start of a given spike.
	\item We sample an $h$-dimensional vector $v$ for a multivariate Gaussian distribution with mean $0$ and covariance matrix $B$, where $B_{ij}=\frac{i(h+1-j)}{h+1}$ for $i\leq j$. Thus $v$ is a realization of a \emph{Brownian bridge} process.
	\item We update $\hat y_{\ell+i}\leftarrow \hat y_{\ell+i}+|v_i|$.  
\end{enumerate}
Note that two different spikes may overlap, in which case the true signal $\hat y$ would have a single spike with larger intensity. Also note that the true signal $\hat y$ generated in this way has at most $hs$ non-zeros and at most $s$ spikes, but may have fewer if overlaps occur. Then, given a noise parameter $\sigma$, we generate the noisy observations $y_i=\hat y_i+\varepsilon_i$, where $\varepsilon_i$ follows a truncated normal distribution with mean $0$, variance $\sigma_i^2$ and lower bound $-\hat y_i$. Finally, we scale the data so that $\|y\|_\infty=1$.

\subsubsection{Methods} We compare the following methods:
\begin{description}
	\item[L1] Corresponds to solving the \lasso problem 
	\begin{align*}
	\min_{x\in \R_+^n}\|y-x\|_2^2+\lambda\sum_{i=1}^{n-1}(x_{i+1}-x_i)^2+\mu\|x\|_1.
	\end{align*}
	
	\item[Decomp-sparse] 
	Enforces the prior that the signal has a most $hs$ non-zeros, by solving the convex optimization problem
	\begin{align*}
	\min_{x\in \R_+^n, z\in [0,1]^n}\;&\|y\|_2^2-2\sum_{i=1}^ny_ix_i+\theta(z,x)+\mu\|x\|_1\\
	\text{s.t.}\;& \sum_{i=1}^n z_i\leq hs\\
	& 0\leq x\leq \|y\|_\infty z
	\end{align*}
	using Algorithm~\ref{alg:Mmatrix}.
	\item[Decomp-prior] In addition to the sparsity prior as before, it incorporates the information that the underlying signal has a most $s$ spikes and that each spike has at least $h$ non-zeros. These two priors can be enforced by solving the optimization problem
	\begin{subequations}\label{eq:L0prior}
	\begin{align}
	\min_{x\in \R_+^n, z\in [0,1]^n}\;&\|y\|_2^2-2\sum_{i=1}^ny_ix_i+\theta(z,x)+\mu\|x\|_1\\
	\text{s.t.}\;& \sum_{i=1}^n z_i\leq hs\\
	& \sum_{i=1}^{n-1}|z_{i+1}-z_i|\leq 2s\label{eq:L0prior_spike}\\
	& \sum_{i=\max\{1,\ell-h\}}^{\min\{n,\ell+h\}}z_i\geq h z_\ell&\ell=1,\ldots,n\label{eq:L0prior_dense}\\
	& 0\leq x\leq  \|y\|_\infty z. 
	\end{align}
	\end{subequations}
Constraint \eqref{eq:L0prior_spike} states that the process can transition from a zero value to a non-zero value at most $2s$ times, thus can have at most $s$ spikes. Each constraint  \eqref{eq:L0prior_dense} states that, if $z_\ell=1$, then there must be at least $h-1$ neighboring non-zero points, thus non-zero indexes occur in patches of at least $h$ elements.
\end{description}
Observe that, following the results in \cite{mazumder2017subset}, we keep an $\ell_1$-regularization for shrinkage to improve performance in low signal-noise-ratio regimes. 

\subsubsection{Computational setting} For the computations in this section, we generate instances with $n=1,000$, $s=10$, and $h=10$; so, each signal is zero in approximately 90\% of the time. Moreover, we test noise levels 
 $\sigma=0.1t$, $t=1, \ldots, n$, and for each $\sigma$ we generate 10 different instances as follows:
\begin{enumerate}
	\item For each parameter combination, 
	two signals are randomly generated: one signal for training, the other for testing. 
	\item For all methods, we solve the corresponding optimization problem for the training signal with 10 values of the smoothness parameter $\lambda$ and 10 values of the shrinkage parameter $\mu$, a total of 100 combinations. We consider two criteria for choosing a pair $(\lambda,\mu)$:
	\begin{description}
		\item[Error] The pair that best fits the true signal with the respect to the estimation error, i.e.,
		combination minimizing $\|\hat y -x^*\|_2^2$, where $x^*$ is the solution for corresponding optimization. 
		\item[Sparsity] The pair that best matches the sparsity pattern of the true signal\footnote{A point $x_i$ is considered non-zero if $|x_i|>10^{-3}$.}, i.e., combination minimizing  $\sum_{i=1}^n\big||\hat y_i|_0-|x_i^*|_0\big|$. This setting is of practical interest in cases where the training data is partially labeled: the location of the spikes is known but the actual value of the signal is not. 
	\end{description}
		\item We solve the optimization problem for the testing signal with parameters $(\lambda,\mu)$ chosen in (2), and report the results (averaged over the 10 instances). 
\end{enumerate}
For the instances considered, 
Table~\ref{tab:snrs} shows the average Signal-to-Noise Ratio (SNR) as a function of $\sigma$, computed as $\text{SNR}=\frac{\|\hat y\|_2^2}{\|\hat y - y\|_2^2} \cdot$ 
\begin{table}[!h]
	\begin{center}
		\caption{Signal-to-Noise Ratio for different values of the noise.}\label{tab:snrs}
		\begin{tabular}{ c | c c c c c c c c c c}
			\hline
			$\sigma$ & 0.1 & 0.2 & 0.3 & 0.4 & 0.5 & 0.6 & 0.7 & 0.8 & 0.9 & 1.0\\
			\hline
			SNR & 2,200 & 138 & 27 & 8.6 & 3.5 & 1.7 & 0.9 & 0.5 & 0.3 & 0.2\\
			\hline
		\end{tabular}
	\end{center}
\end{table}

\subsubsection{Results with respect to the error criterion}\label{sec:computations_error} We now present the results when the true values of $\hat y$ are known in training. Figure~\ref{fig:error1} depicts the out-of-sample error of each method and SNR, computed as $\text{error}=\frac{\|\hat y_{test}-x^*\|_2^2}{\|\hat y_{test}\|_2^2}$ where $\hat y_{test}$ is the true testing signal, and $x^*$ is the estimator.  Figure~\ref{fig:pattern1} depicts how accurately the estimator obtained in testing matches the sparsity pattern of the true signal. We observe that the standard $\ell_1$-norm approach results in dense signals with a substantial number of false positives, and is outperformed by the approaches that enforce priors in terms of error as well. The inclusion of the sparsity prior results in a notable improvement in terms of the error across all SNRs, and reducing it by half or more for SNR$\geq 3$. This prior also yields an order-of-magnitude improvement in terms of matching the sparsity pattern for SNR$\geq 3$, although for low SNRs the improvement in matching the sparsity pattern is less pronounced (and is worse for SNR=0.5). The inclusion of additional priors for the number and length of each spike yields further improvements (especially for low SNRs), and yields a good match for the sparsity pattern in all cases.

\begin{figure}[h!]
	\begin{center}
		\includegraphics[width=0.80\textwidth,trim={11cm 6cm 11cm 5.5cm},clip]{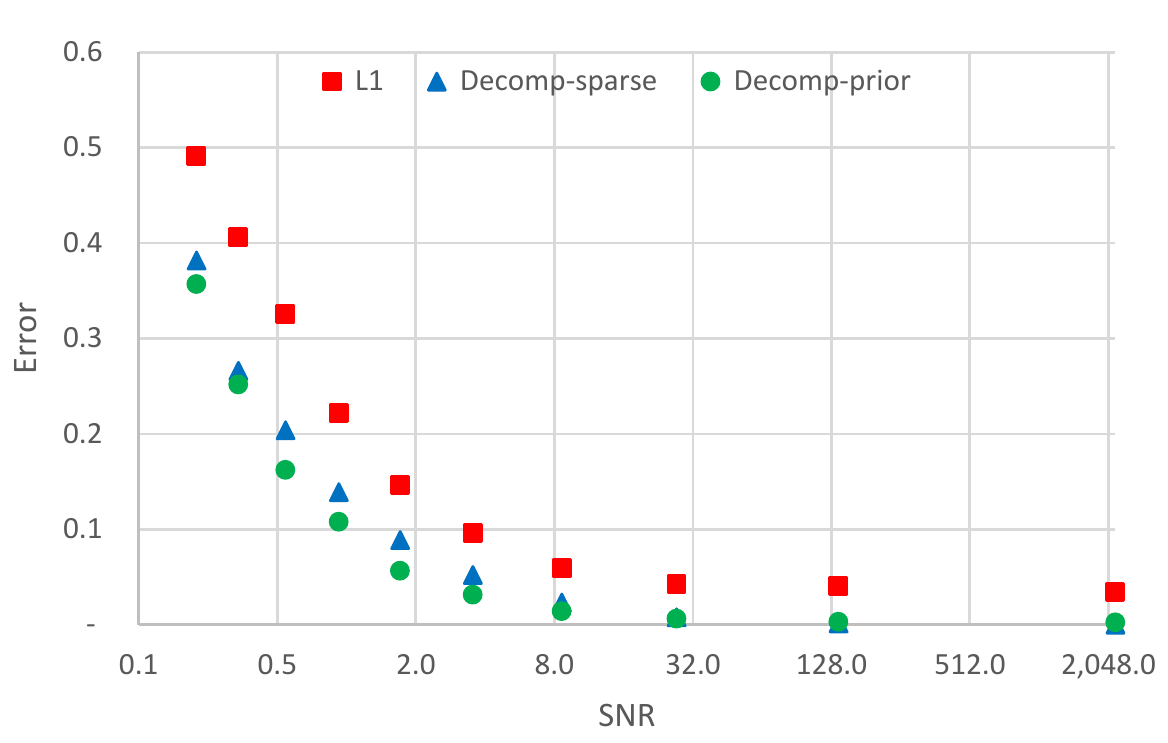}
		\caption{Average out-of-sample error as a function of SNR (in log-scale). }
		\label{fig:error1}
		\end{center}
\end{figure}

\begin{figure}[h!]
	\begin{center}
		\includegraphics[width=0.80\textwidth,trim={11cm 6cm 11cm 5.5cm},clip]{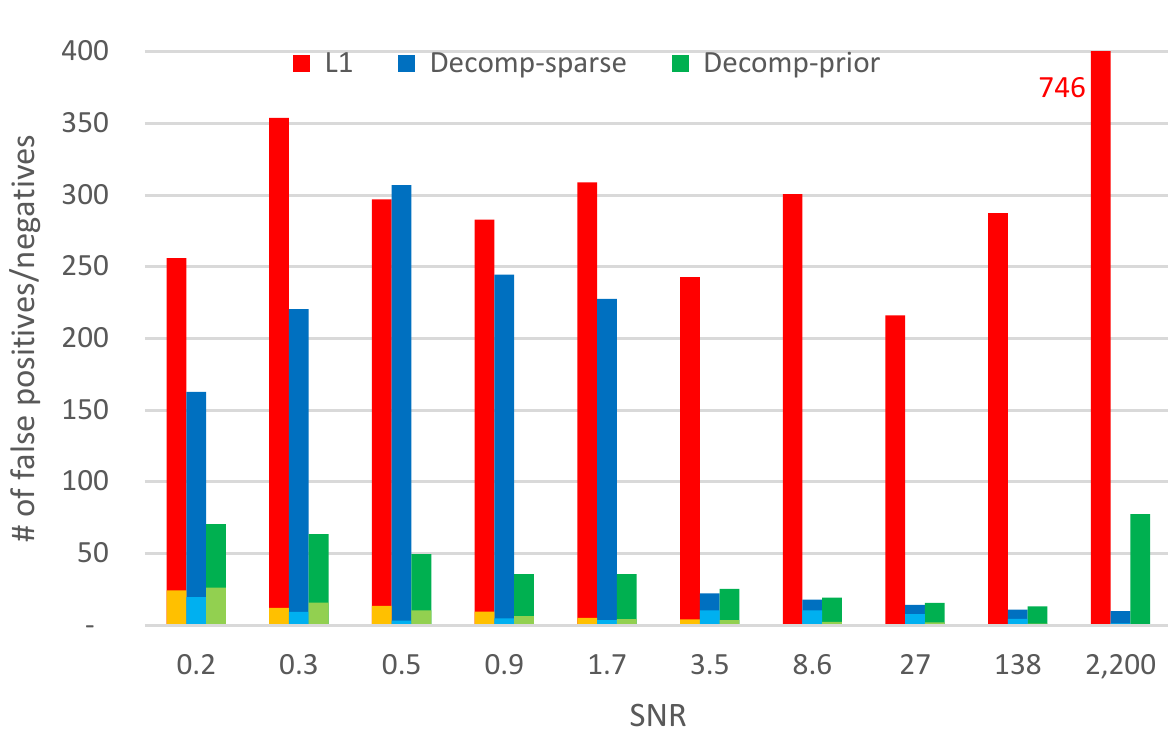}
		\caption{Average out-of-sample number of false positives (red/dark blue/dark green) and false negatives (orange/light blue/light green) as a function of SNR. 
			The number of false positives for L1 with SNR=2,200 is 746.}
		\label{fig:pattern1}
	\end{center}
\end{figure}

Figure~\ref{fig:distribution1} provides detailed information about the distribution of the out-of-sample errors for three different SNRs. We see that in high SNR regimes, the inclusion of the sparsity prior consistently outperforms the $\ell_1$-norm method, and the inclusion of additional priors consistently outperforms using only the sparsity prior. In contrast, in low SNR regimes, while the inclusion of additional priors yields better results on average, the improvement is not as consistent. 

\begin{figure}[h!]
		\subfloat[SNR=0.2.]{\includegraphics[width=0.33\textwidth,trim={10.5cm 5cm 10.5cm 5.5cm},clip]{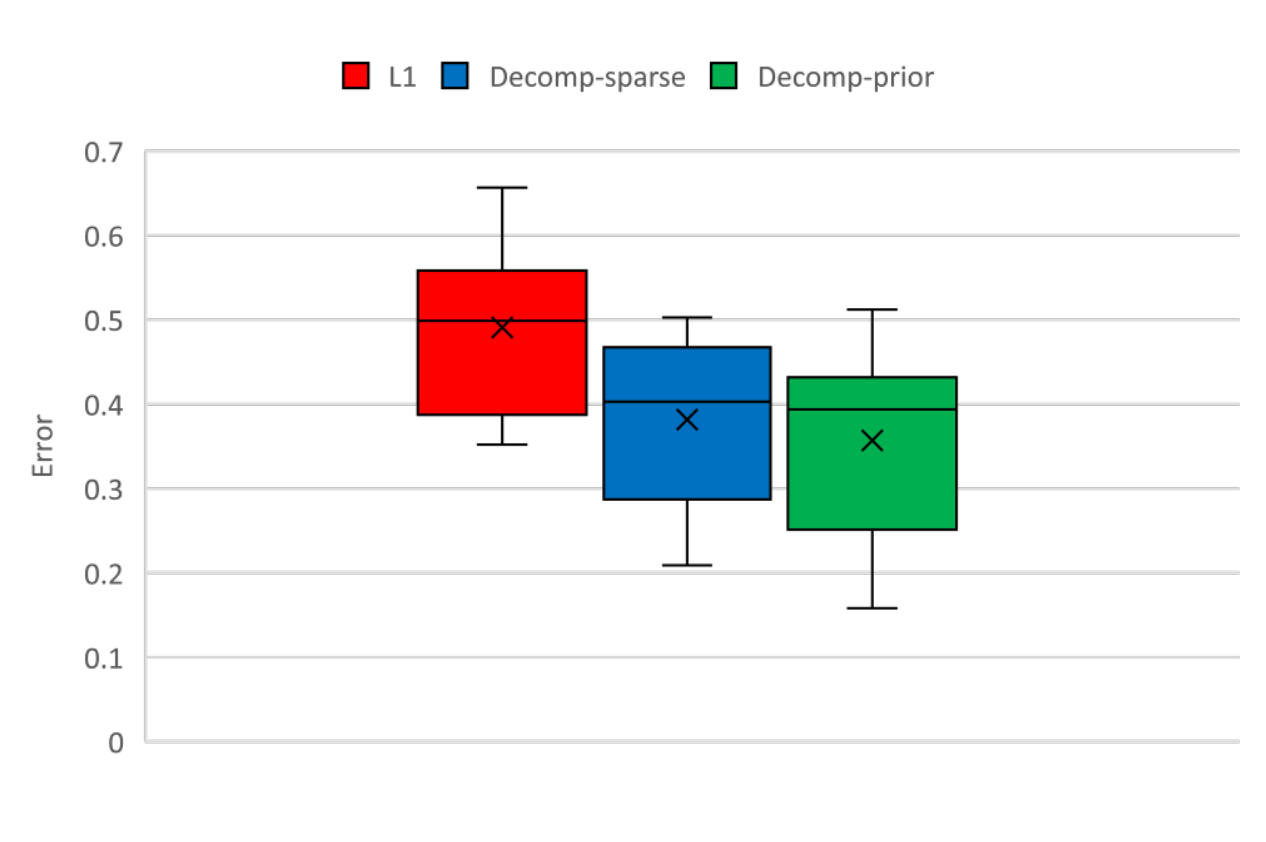}}\hfill		\subfloat[SNR=0.5.]{\includegraphics[width=0.33\textwidth,trim={10.5cm 5cm 10.5cm 5.5cm},clip]{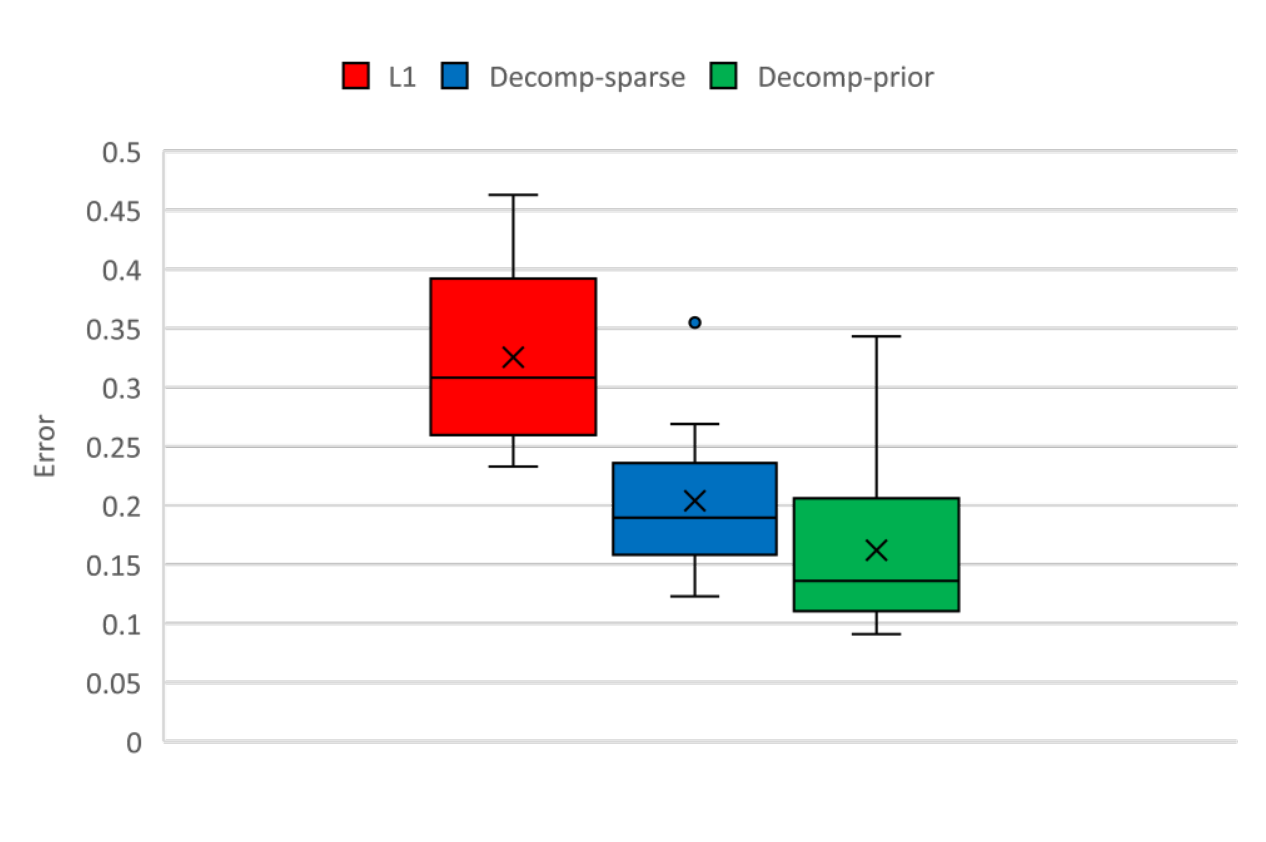}}\hfill
		\subfloat[SNR=3.5.]{\includegraphics[width=0.33\textwidth,trim={10.5cm 5cm 10.5cm 5.5cm},clip]{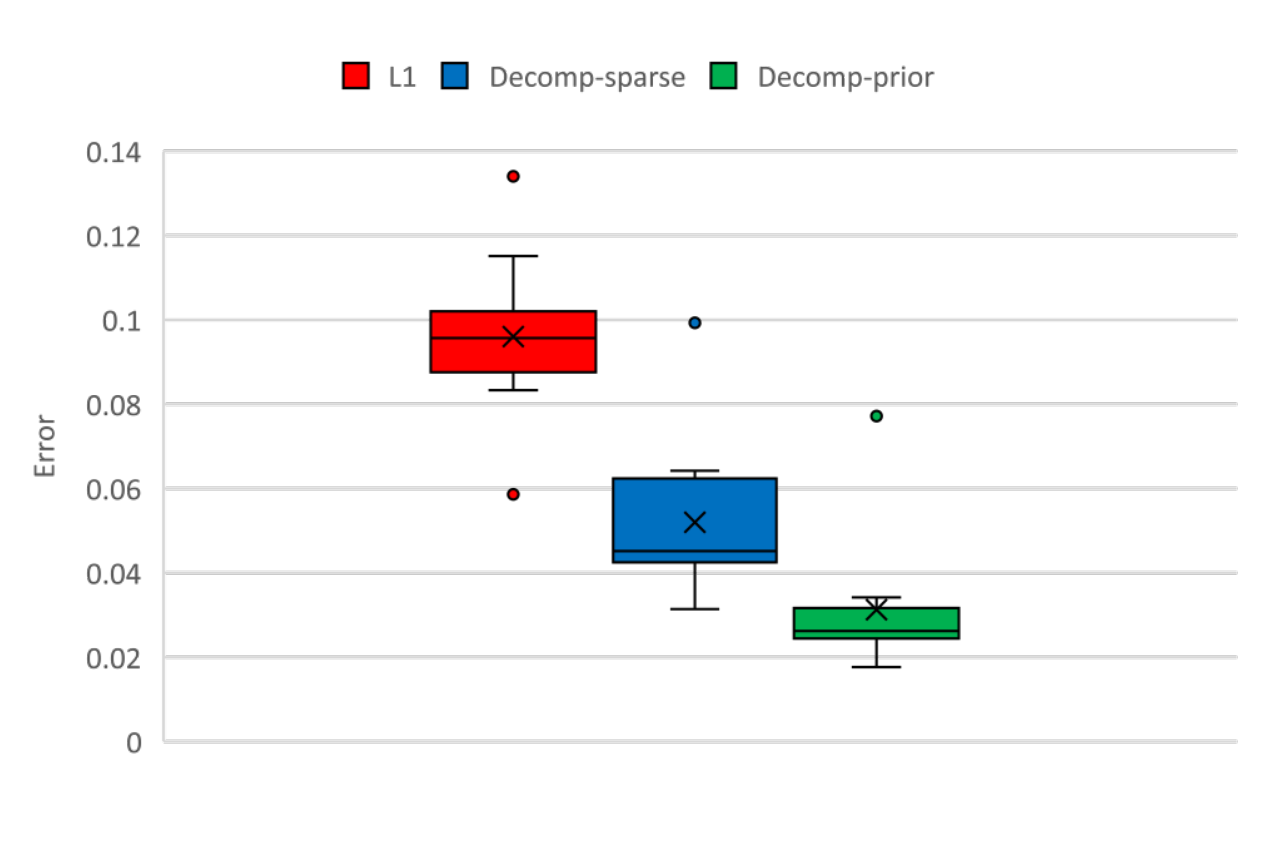}}
		\caption{Distribution of the out-of-sample errors for different SNRs when the true values of the signal used in training are available.}
		\label{fig:distribution1}
\end{figure}

Finally, Figure~\ref{fig:time1} depicts the average time required to solve the optimization problems as a function of the SNR. As expected, the $\ell_1$-norm approximation is the fastest method. Optimization problems with the sparsity prior are solved under two seconds, and optimization problems with all priors are solved under 10 seconds. We see that time required to solve the problems based on the stronger relaxations increases as the SNR decreases. 

\begin{figure}[h!]
	\begin{center}
		\includegraphics[width=0.80\textwidth,trim={11cm 6cm 11cm 5.5cm},clip]{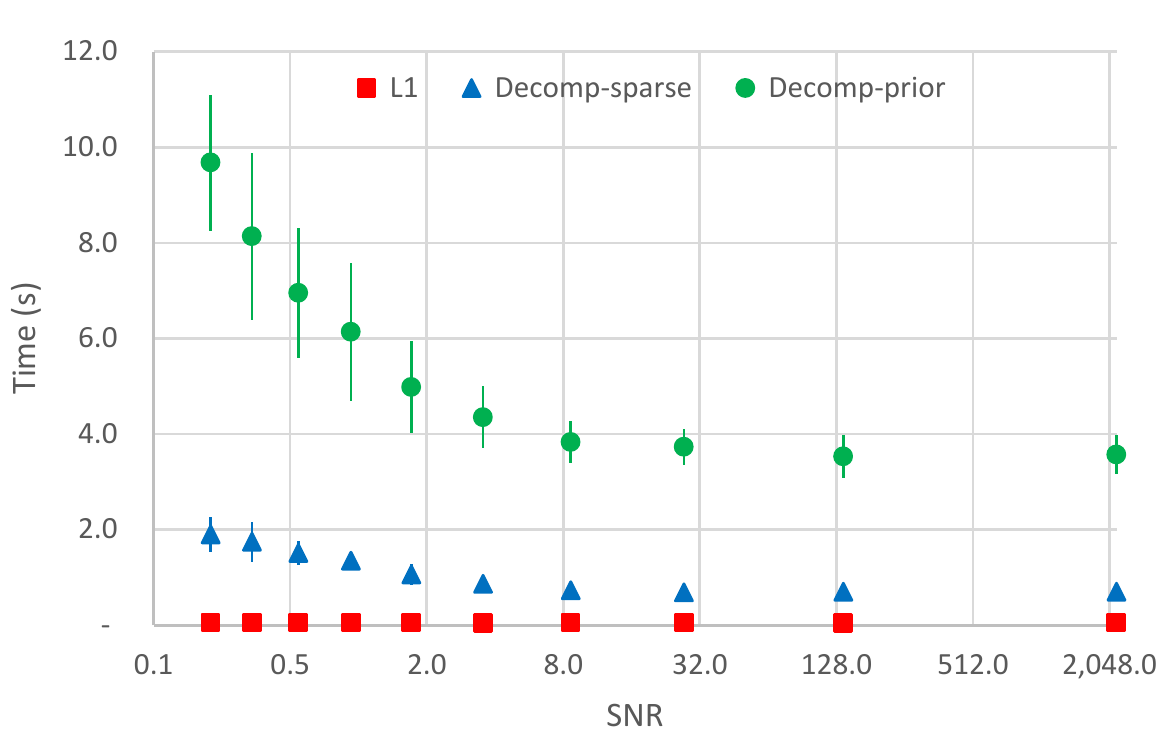}
		\caption{CPU time in seconds as a function of SNR (in log-scale). 
			The error bars correspond to $\pm 1$ stdev.}
		\label{fig:time1}
	\end{center}
\end{figure}

\subsubsection{Results with respect to the sparsity pattern criterion} We now present the results when,
for the training data,
 the true values of $\hat y$ are unknown, but its sparsity pattern is known. Figure~\ref{fig:error2} depicts the out-of-sample error of each method for each SNR and Figure~\ref{fig:pattern2} depicts how accurately the estimator obtained in validation matches the sparsity pattern of the true signal. Naturally, as the true values of the training signal are unknown, all methods perform worse in terms of the out-of-sample error. The $\ell_1$-norm method in particular performs very poorly in low SNR regimes: the estimator is $x\approx 0$, resulting a large error close to one and several false negatives (with no false positives, since few or no indexes are non-zero). In contrast the methods that enforce priors result in significantly reduced error across all SNRs while simultaneously improving the detection of the sparsity in low SNRs regimes, correctly detecting several spikes. In this setting, we did not observe a substantial difference between methods Decomp-sparse and Decomp-prior. From Figure~\ref{fig:distribution2},which depicts the distributions of the errors, we see that the new convexification-based methods consistently outperform the $\ell_1$-method.

\begin{figure}[h!]
	\begin{center}
		\includegraphics[width=0.80\textwidth,trim={11cm 6cm 11cm 5.5cm},clip]{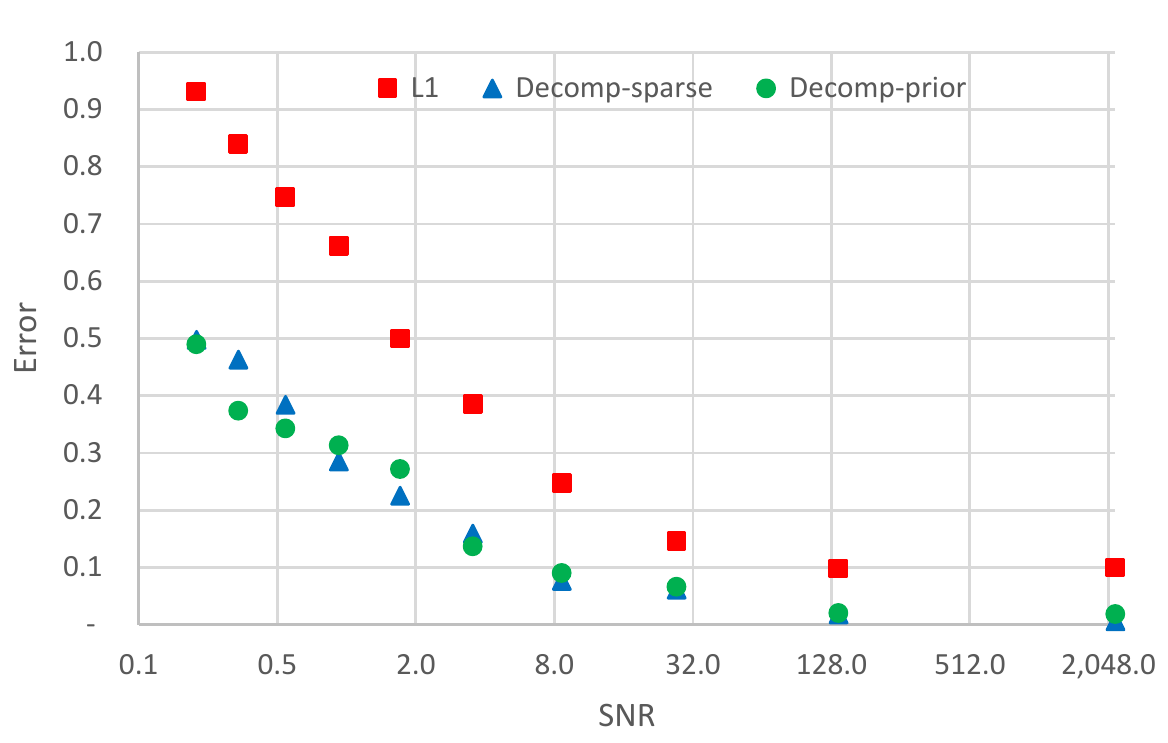}
		\caption{Average out-of-sample error as a function of SNR (in log-scale) when only the sparsity pattern of the training signal is known.}
		\label{fig:error2}
	\end{center}
\end{figure}
\begin{figure}[h!]
	\begin{center}
		\includegraphics[width=0.80\textwidth,trim={11cm 6cm 11cm 5.5cm},clip]{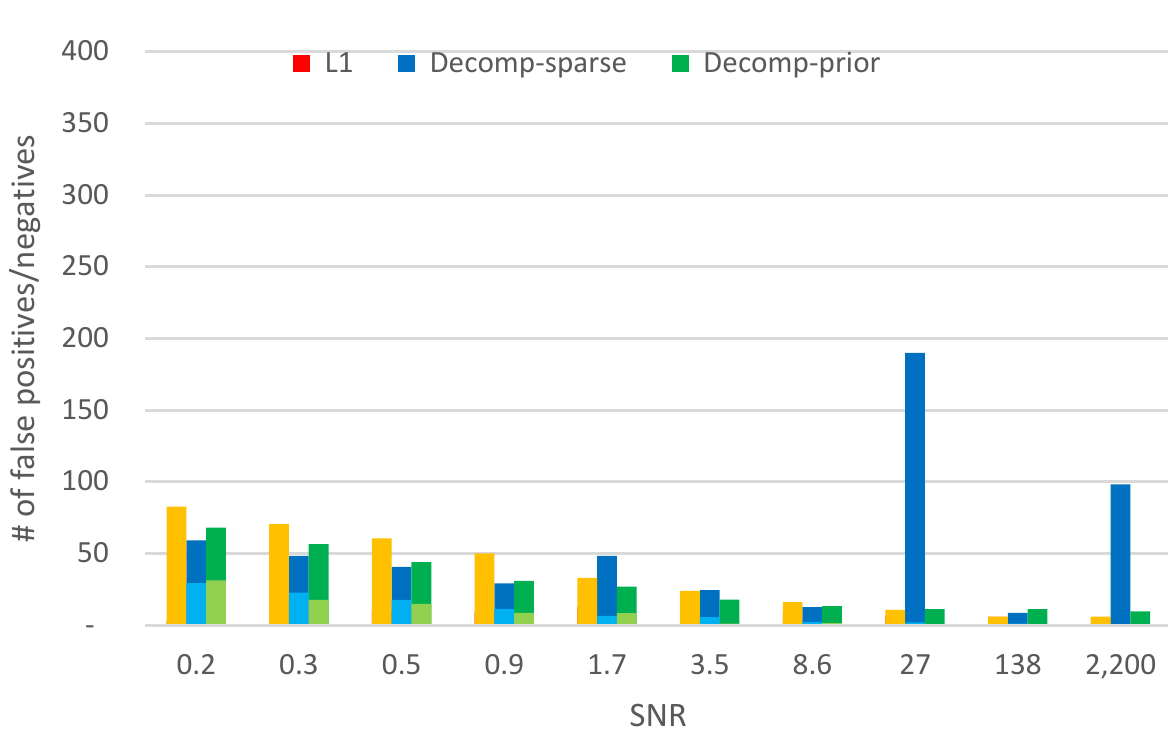}
		\caption{Average out-of-sample number of false positives (red/dark blue/dark green) and false negatives (orange/light blue/light green) as a function of SNR when only the sparsity pattern of the training signal is known. }
		\label{fig:pattern2}
	\end{center}
\end{figure}

\begin{figure}[h!]
	\subfloat[SNR=0.2.]{\includegraphics[width=0.33\textwidth,trim={10.5cm 5cm 10.5cm 5.5cm},clip]{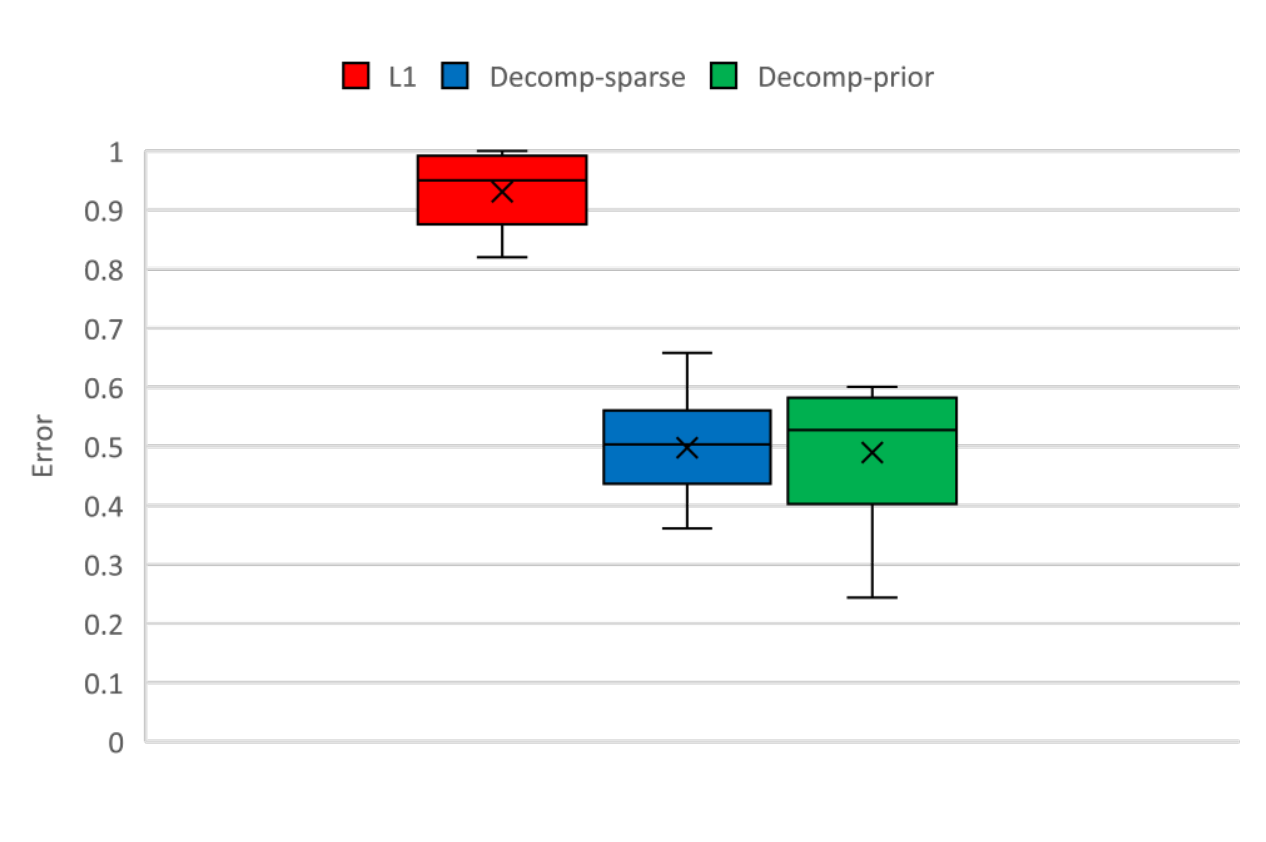}}\hfill		\subfloat[SNR=0.5.]{\includegraphics[width=0.33\textwidth,trim={10.5cm 5cm 10.5cm 5.5cm},clip]{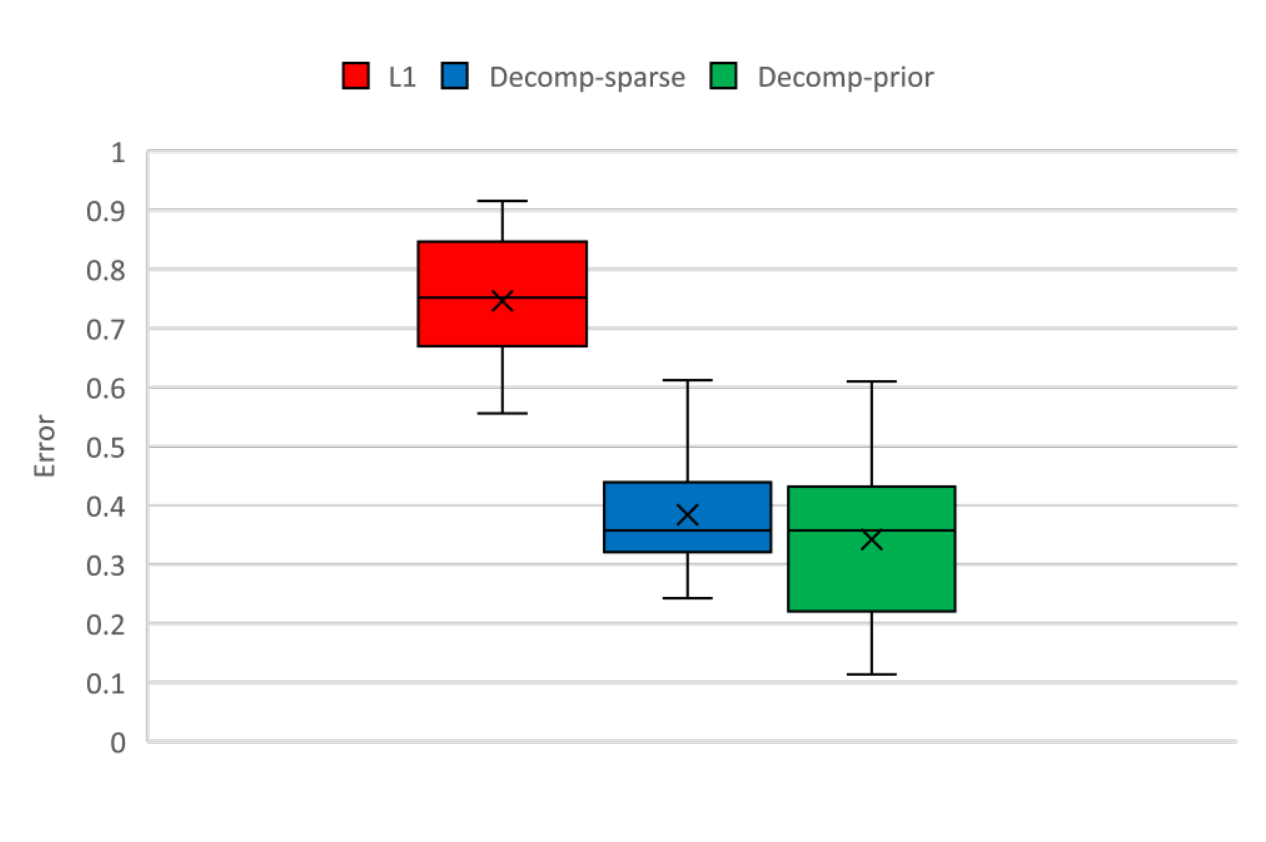}}\hfill
	\subfloat[SNR=3.5.]{\includegraphics[width=0.33\textwidth,trim={10.5cm 5cm 10.5cm 5.5cm},clip]{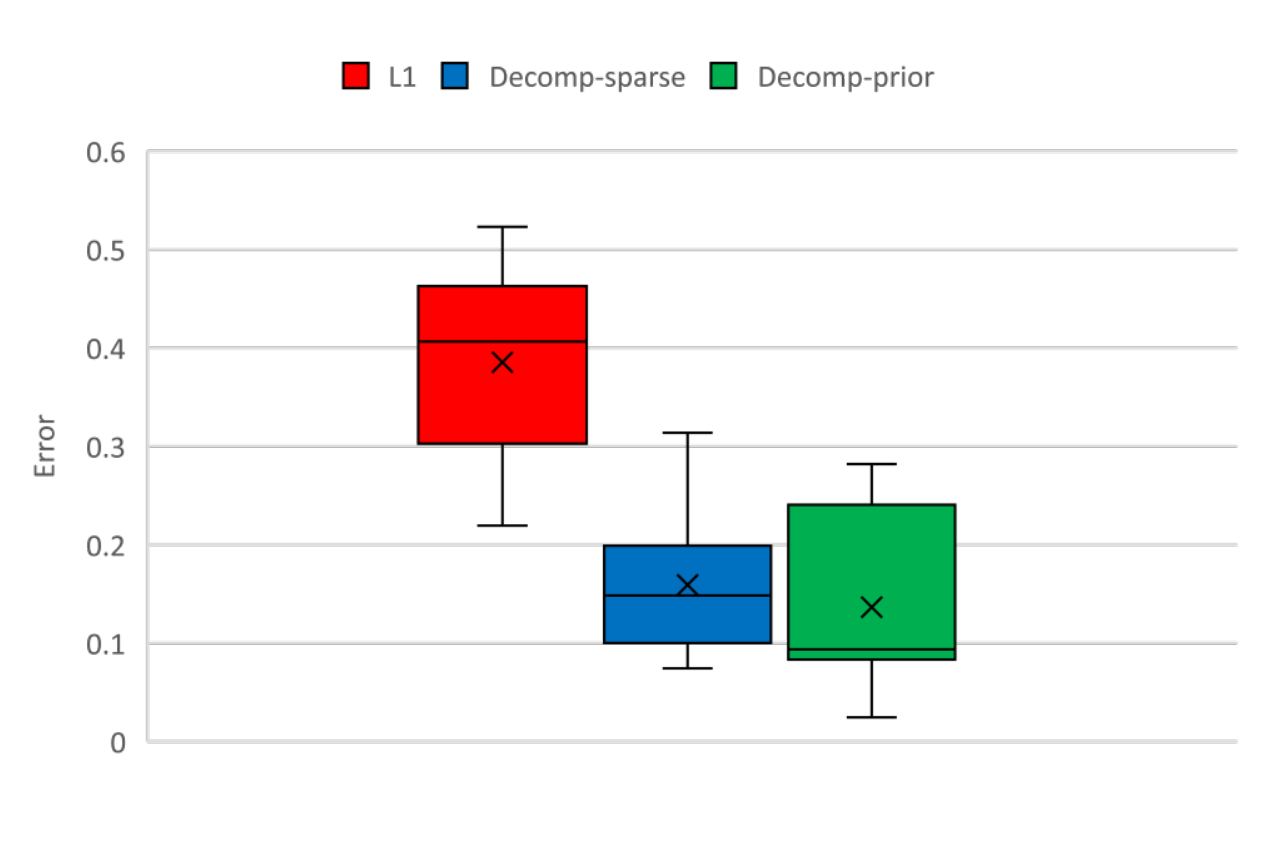}}
	\caption{Distribution of the out-of-sample errors for different SNRs when only the sparsity pattern of the training signal is known.}
	\label{fig:distribution2}
\end{figure}

\subsection{Computational experiments - Lagrangian methods} \label{sec:comp_decomposition}We now report on the performance of the Lagrangian method given in Section~\ref{sec:decomposition} for larger signals with $n=100,000$, $\sigma=0.5$, $s=10$ and $h=100$ (so approximately 1\% of the signal values are non-zero). 
We denoise the signal by solving the optimization problem 
	\begin{subequations}\label{eq:L0lagrange}
	\begin{align}
	\min_{x\in \R_+^n, z\in [0,1]^n}\;&\|y\|_2^2-2\sum_{i=1}^ny_ix_i+\theta(z,x)+\mu\|x\|_1+\kappa\|z\|_1\\
	\text{s.t.}\;& 0\leq x\leq \|y\|_\infty z.
	\end{align}
\end{subequations}
In these experiments we use synthetic instances generated as in Section~\ref{sec:syntData} with $\lambda=0.3$ and $\mu=0$\footnote{In the experiments reported in Section~\ref{sec:computations_error} with $\sigma=0.5$ and method decomp-sparse, the combination $(\lambda,\mu)=(0.32,0)$ was chosen in 4/10 instances and was the combination more often selected in training.} and varying $\kappa\in \{0.0005, 0.001,0.002,0.005,0.01,0.02\}$. We solve \eqref{eq:L0lagrange} using the Lagrangian method with $m\in \{1,10,100,1000\}$ subproblems ($m=1$ corresponds to no decomposition). The independent subproblems are solved in parallel in the same laptop computer.

Table~\ref{tab:decomp} presents the results, both in terms of statistical and computational performance. For each value of $\kappa$ and $m$, it shows the error between the true signal and the estimated signal\footnote{Since we do perform cross-validation, we report the in-sample error.}, the number of non-zero values $\|x^*\|_0$ of the resulting estimator, the time required to solve the problem; the number of subgradient iterations used, and the actual number of subproblems solved.

We observe that for the smallest value of $\kappa=0.0005$ (corresponding to a sparsity of $\|x^*\|_0\approx 30,000$), the method without decomposition ($m=1$) is the fastest and is able to solve the problems in approximately three minutes. However, as the value of the $\ell_0$ regularization parameter $\kappa$ increases, the Lagrangian methods solve the problems increasingly faster. In particular, for values of $\kappa\geq 0.005$ (sparsity of $\|x^*\|_0\leq 900$), the Lagrangian method with $m=1,000$ solves the problems in under one minute whereas a direct implementation via Algorithm~\ref{alg:Mmatrix} may require an hour or more. Indeed, we see that as the $\ell_0$ regularization parameter increases, the number of iterations and number of subproblems solved decreases considerably. In fact, if $\kappa\geq 0.01$, the Lagrangian method with $m=10$ is solved to optimality without performing any subgradient iterations. Finally, we point out that in terms of the estimation error, all methods return comparable errors (except for $\kappa=0.0005$, where the maximum number of 100 iterations is reached and the Lagrangian methods do not solve the problems to optimality). 

Therefore, we conclude that the proposed Lagrangian method is able to efficiently tackle large-scale problems when the target sparsity is small compared to the dimension of the problem, and can solve the problems  by two-orders of magnitude faster compared to default method. The drawback is that the decomposition method is unable to incorporate additional priors using constraints.
}

\begin{table}
	\begin{center}
		\caption{Performance of the Lagrangian method for signals with $n=100,000$. \textbf{Bold entries} correspond to the best estimation error and the fastest solution time.}
		\label{tab:decomp}
		
		\begin{tabular}{c c | l l | l l l}
			\hline
			\multirow{2}{*}{$\kappa$} & \multirow{2}{*}{$m$} &\multicolumn{2}{c|}{\underline{signal quality}}&\multicolumn{3}{c}{\underline{computational performance}}\\
			&& error& $\|x^*\|_0$& time &\# iter & \# sub\\
			\hline
			\multirow{4}{*}{0.0005}&1&\textbf{0.172} & 29,690 & \textbf{172} & 1 & 1\\
			& 10 & 0.188 & 30,650 & 1,825 & 89 & 10+362\\
			&100 & 0.187 & 30,648 & 656 & 100 & 100+3,252\\
			& 1,000 & 0.187 & 30,670 & 631 & 100 & 1,000+31,932\\
			&&&&&&\\
				\multirow{4}{*}{0.001}&1&0.091 & 10,790 & 506 & 1 & 1\\
			& 10 & 0.091 & 10,789 & 3,391 & 62 & 10+257\\
			&100 & 0.091 & 10,781 & 613 & 100 & 100+2,455\\
			& 1,000 & \textbf{0.090} & 10,760 & \textbf{475} & 100 & 1,000+23,174\\
			&&&&&&\\	\multirow{4}{*}{0.002}&1&\textbf{0.027} & 2,523 & 1,390 & 1 & 1\\
			& 10 & \textbf{0.027} & 2,511 & 1,703 & 22 & 10+94\\
			&100 & \textbf{0.027} & 2,510 & 280 & 100 & 100+848\\
			& 1,000 & \textbf{0.027} & 2,502 & \textbf{173} & 100 & 1,000+7,861\\
			&&&&&&\\	\multirow{4}{*}{0.005}&1&\textbf{0.008} & 878 & 5,579 & 1 & 1\\
			& 10 & \textbf{0.008} & 877 & 309 & 12 & 10+20\\
			&100 & \textbf{0.008} & 877 & 51 & 17 & 100+54\\
			& 1,000 & \textbf{0.008} & 878 & \textbf{31} & 61 & 1,000+480\\
			&&&&&&\\	\multirow{4}{*}{0.01}&1&\textbf{0.013} & 758 & 2,141 & 1 & 1\\
			& 10 & \textbf{0.013} & 758 & 174 & 1 & 10+0\\
			&100 & \textbf{0.013} & 759 & 81 & 18 & 100+38\\
			& 1,000 & \textbf{0.013} & 761 & \textbf{49} & 71 & 1,000+347\\
			&&&&&&\\	\multirow{4}{*}{0.02}&1&\textbf{0.028} & 648 & 2,184 & 1 & 1\\
			& 10 & 0.030 & 637 & 185 & 1 & 10+0\\
			&100 & \textbf{0.028} & 646 & 89 & 14 & 100+27\\
			& 1,000 & \textbf{0.028} & 649 & \textbf{44} & 62 & 1,000+275\\
			\hline
		\end{tabular}
	\end{center}
\end{table}

\section{Conclusions}
\label{sec:conclusions}

In this paper we derived strong iterative convex relaxations for quadratic optimization problems with M-matrices and indicators, of which signal estimation with smoothness and sparsity is a special case. The relaxations are based on convexification of quadratic functions on two variables, and optimal decompositions of an M-matrix into pairwise terms. We also gave extended conic quadratic formulations of the convex relaxations, allowing the use of off-the-shelf conic solvers. The approach is general enough to permit the addition of multiple priors in the form of additional constraints. The proposed iterative convexification approach substantially closes the gap between the $\ell_0$-``norm" and its $\ell_1$ surrogate and results in significantly better estimators than the \rev{standard} approaches \rev{using $\ell_1$ approximations}. In fact, near-optimal solution of the $\ell_0$-problems are obtained in seconds for instances with over 10,000 variables, \rev{and the method scales to instances with 100,000 variables using tailored algorithms}. 

In addition to better inference properties, the proposed models and resulting estimators are easily \emph{interpretable}. On the one hand, unlike \rev{$\ell_1$-approximations} and related estimators, the sparsity of the proposed estimators is close to the target sparsity parameter $k$. Thus, a prior on the sparsity of the signal can be naturally fed to the inference problems. On the other hand, the proposed strong convex relaxations compare favorably to \rev{$\ell_1$-approximations} in classification or spike inference purposes: the 0-1 variables can be easily used to assign a category to each observation via simple rounding heuristics, and resulting in high-quality solutions. 

\section*{Acknowledgments}

A. Atamt\"urk is supported, in part, by grant FA9550-10-1-0168 from the Office
of the Assistant Secretary of Defense for Research and Engineering and grant 1807260 from the National Science Foundation. A. G\'omez is supported, in part, by the National Science Foundation under Grant No. 1818700.

\bibliographystyle{apalike}
\bibliography{Bibliography}

\end{document}